\documentclass[journal]{IEEEtran}
\IEEEoverridecommandlockouts
\makeatletter
\def\endthebibliography{%
  \def\@noitemerr{\@latex@warning{Empty `thebibliography' environment}}%
  \endlist
}
\makeatother

\usepackage{cite}
\usepackage{amsmath,amssymb,amsfonts}
\usepackage{algorithm}
\usepackage{algpseudocode}
\usepackage{graphicx}
\usepackage[caption=false]{subfig}
\usepackage{textcomp}
\usepackage{float}
\usepackage{stfloats}
\usepackage{xcolor}
\usepackage{balance}
\usepackage{dsfont}
\usepackage{multirow}
\usepackage{ntheorem}
\usepackage[colorlinks,
            linkcolor=black,
            anchorcolor=black,
            citecolor=black]{hyperref}

\newtheorem{theorem}{\textbf{Theorem}}
\newtheorem*{proof}{Proof}
\newtheorem{lemma}{\textbf{Lemma}}
\newtheorem{remark}{Remark}
\newtheorem{assumption}{Assumption}
\newtheorem{definition}{Definition}
\newtheorem{proposition}{\textbf{Proposition}}
\newtheorem{corollary}{\textbf{Corollary}}

\newcommand{\RNum}[1]{\uppercase\expandafter{\romannumeral #1\relax}}

\ifodd 1
\newcommand{\com}[1]{{\color{black}#1}} 
\else

\newcommand{\com}[1]{}
\fi

\def\BibTeX{{\rm B\kern-.05em{\sc i\kern-.025em b}\kern-.08em
    T\kern-.1667em\lower.7ex\hbox{E}\kern-.125emX}}

\begin{document}
\title{{A Federated Online Restless Bandit Framework for Cooperative Resource Allocation}  \\
}
\author{\IEEEauthorblockN{Jingwen Tong, ~\IEEEmembership{Member,~IEEE},
Xinran Li,
Liqun Fu, ~\IEEEmembership{Senior Member,~IEEE},
Jun Zhang,  ~\IEEEmembership{Fellow,~IEEE},
 and
Khaled B. Letaief, ~\IEEEmembership{Fellow,~IEEE}}
\thanks{
Jingwen Tong, Xinran Li, Jun Zhang, and Khaled B. Letaief
are with the Department of Electronic and Computer Engineering,
The Hong Kong University of Science and Technology (HKUST), Kowloon, Hong Kong (e-mails: eejwentong@ust.hk; xinran.li@connect.ust.hk; eejzhang@ust.hk; eekhaled@ust.hk).
Liqun Fu is with the School of Informatics, Xiamen University, Xiamen 361005, China
(e-mail: liqun@xmu.edu.cn). The corresponding author is Jun Zhang (eejzhang@ust.hk)
}
}	

\maketitle

\begin{abstract}
Restless multi-armed bandits (RMABs) have been widely utilized to address resource allocation problems with Markov reward processes (MRPs).
Existing works often assume that the dynamics of MRPs are known prior,
which makes the RMAB problem solvable from an optimization perspective.
Nevertheless, an efficient learning-based solution for RMABs with unknown system dynamics remains an open problem.
In this paper, we study the cooperative resource allocation problem with unknown system dynamics of MRPs.
This problem can be modeled as a multi-agent online RMAB problem, where multiple agents collaboratively learn the system dynamics while maximizing their accumulated rewards.
We devise a federated online RMAB framework to mitigate the communication overhead and data privacy issue by adopting the federated learning paradigm.
Based on this framework, we put forth a Federated Thompson Sampling-enabled Whittle Index (FedTSWI) algorithm to solve this multi-agent online RMAB problem.
The FedTSWI algorithm enjoys a high communication and computation efficiency, and a privacy guarantee.
Moreover, we derive a regret upper bound for the FedTSWI algorithm.
Finally, we demonstrate the effectiveness of the proposed algorithm on the case of online multi-user multi-channel access.
Numerical results show that the proposed algorithm achieves a fast convergence rate of $\mathcal{O}(\sqrt{T\log(T)})$ and better performance compared with baselines.
More importantly, its sample complexity decreases with the number of agents.
\end{abstract}


\section{Introduction}
Cooperative resource allocation refers to distributing shared resources or developing allocation policies among multiple agents collaboratively \cite{chen2008unified}.
This process requires decisions on resource allocation that account for the needs and priorities of all involved agents.
Cooperative resource allocation typically arises in situations where resources are limited, and each agent only has partial
information of the global environment due to geographic or service heterogeneity.
Consequently, an optimal allocation policy necessitates aggregating observations from all agents at a central server.
Examples of shared resources in wireless communications include energy \cite{xie2013optimal}, bandwidth\cite{han2008resource}, channel \cite{li2013efficient}, and space \cite{tong2021throughput}. Fig. \ref{RelFig} depicts the challenges, features, and potential solutions of the cooperative resource allocation problems in next-generation wireless communication with the increasing resource shortage, complex environment, and high-density network \cite{letaief2019roadmap}.
This kind of cooperation resource allocation problem often requires making decisions in a dynamic environment, which  aligns with the concept of online learning in sequential decision-making frameworks.

Multi-armed bandit (MAB) is a type of sequential decision-making frameworks for balancing the exploration and exploitation dilemma \cite{bubeck2012regret}.
Among the MABs, Restless MAB (RMAB) attracts great attention in resource allocation problems with dynamic environment (or MRPs) \cite{whittle1988restless, nino2023markovian}.
In a basic RMAB, an agent selects $K$ out of $N$ arms to play at each round and
will receive a reward from the environment.
During this process, the state of each arm will change with the MRP's dynamics, whether it was selected or not,
but only the selected arms yield rewards.
The agent's goal is to maximize the accumulated rewards over time horizon $T$.
To solve this RMAB problem, existing works often assume that the system dynamics are known prior
and solve it from an optimization perspective.
\begin{figure}[!t]
\centering
\includegraphics[width=3.5in]{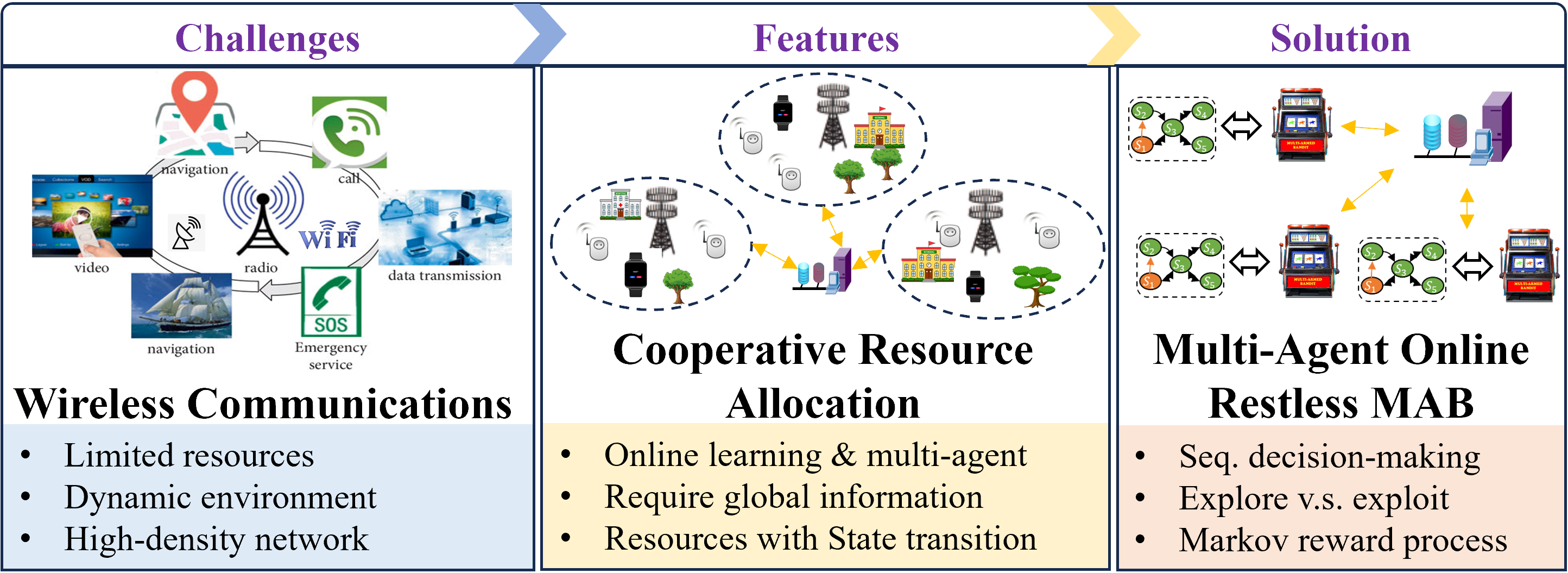}
\caption{Multi-agent RMAB for cooperative resource allocation.}
\label{RelFig}
\end{figure}

However, such prior knowledge is typically unavailable and may be inaccurate in the practical cooperation resource allocation problems \cite{wang2022optimistic}.
This motivates us to study the online RMAB problem,
i.e., to learn the system dynamics online while maximizing the accumulated rewards.
A naive approach is to follow the certainty equivalence principle (e.g., myopic policy) \cite{van1981certainty}, which, however, is suboptimal due to the lack of exploration \cite{ouyang2017learning}.
An effective algorithm needs to not only exploit available
dynamic estimations to maximize its instant rewards but also explore
the unknown dynamics to maximize its long-term rewards.
Many works adopt reinforcement learning (RL) to strike a better exploration-exploitation trade-off.
Unfortunately, their solutions often suffer from high sample complexity and also lack theoretical guarantees \cite{liu2012learning}.
In this paper, we investigate the online RMAB problem from a model-based perspective,
which enjoys both satisfactory performance and high sample efficiency.

In particular, we model the cooperation resource allocation problem as an online partially observable Markov decision process (POMDP)
since the state of the arm is observed only when selected.
Two key challenges exist in solving this online POMDP from the model-based perspective.
First,  the number of system states is uncountable.
As mentioned in \cite{papadimitriou1987complexity}, planning in an MDP is PSPACE-hard due to the state-dependent reward and $N$-dimensional action space.
The continuous belief state in the POMDP makes this issue more severe.
Second, the sample complexity is extremely high.
This is because the collected samples are valid for dynamic estimation only when an arm is selected twice consecutively.
To overcome these challenges,
we further formulate this cooperation resource allocation problem as a multi-agent online RMAB problem.
The multi-agent online RMAB problem can i) reduce the computational complexity by decoupling a $N$-dimension POMDP into $N$ one-dimension belief MDPs using the dynamic allocation index method \cite{gittins2011multi}; ii) increase the sample efficiency in the estimation process by extending the single-agent setting to multi-agent setting.

However, collecting all agents' samples to the central server for model aggregation may result in high communication overhead and data privacy concerns \cite{shao2023survey}.
To overcome this obstacle,
we propose a federated online RMAB framework by combining the RMAB with the federated learning (FL) paradigm \cite{liu2022hierarchical, mcmahan2017communication}.
In this framework, the collected samples are stored at the agent's side, and only system dynamics parameters are sent to the server for model aggregation.
The central server computes the selection policy under the aggregated model and broadcasts it to all agents.
This interaction follows the generalized policy iteration (GPI) principle in the optimal control theory \cite{SuttonRL2018},
i.e., policy evaluation accounts for the dynamic estimation at each agent, and policy improvement is performed at the central server to track the optimal selection policy.
Therefore, the proposed framework is able to pursue the optimal policy for the multi-agent online MAB problem.

Based on the federated online RMAB framework,
we put forth a Federated Thompson Sampling-enabled Whittle Index (FedTSWI) algorithm to solve the multi-agent online RMAB problem.
The proposed algorithm proceeds in episodes where the central server and the agents iteratively track the optimal solution by
following the GPI principle.
Specifically, the central server employs a federated TS algorithm to learn the system dynamics among different agents at each episode.
Then, each agent receives the updated Whittle Index (WI) policy from the central server and executes it for several time slots.
To demonstrate the effectiveness of the FedTSWI algorithm,
we consider the online multi-user multi-channel access problem as a case study,
where multiple secondary users opportunistically access the licensed channels when it is sensed as vacant \cite{letaief2009cooperative}.
In this case, the estimation process of the system dynamics can be modeled as the Beta-Bernoulli distribution.
As a result, each agent only needs to upload the number of times the state-action pair has been selected to the central server.
Therefore, the proposed algorithm has a low implementation complexity.
Moreover, we derive a regret upper bound (i.e., $\mathcal{O}(\sqrt{T\log(T)})$) for the proposed algorithm,
which increases sublinear with the time horizon $T$ when the WI policy is asymptotically optimal.

\subsection{Related Work}
Abundant literature has considered the cooperative resource allocation problems, among which many real-world settings can be further formulated as a multi-agent MAB problem.
Shi  \emph{et al.} \cite{shi2021federated} considered a cognitive radio system where secondary users want to select $K$ out of $N$ channels that are vacant in a coverage area.
The global channel availability at the base station is the average over the entire coverage area due to geographic heterogeneity.
A standard solution is to encourage these secondary users to sample the channels and then aggregate the information at the base station for model aggregation (i.e., the channel availability).
Meanwhile, Yang \emph{et al.} \cite{yang2020federated} considered the recommendation system, where the central server wants to recommend the most popular item to new customers to maximize the expected reward.
The server does not initially have the global item popularity but can learn via customers.
\com{In addition, Xia \emph{et al.} \cite{xia2020multi} and Li \emph{et al.} \cite{li2022privacy} considered the client selection problem in FL by formulating it as a multi-agent stochastic MAB problem. Two upper confidence bound (UCB)-based algorithms have been proposed in these works by considering the large communication overhead, data heterogeneity, and privacy.}
In this paper,  we study cooperative resource allocation under the federated online RMAB framework, which holds great potential for solving various real-world applications.

Solving the RMAB problem with known system dynamics utilizing the optimization methods has been well-studied.
Ortner  \emph{et al.} \cite{ortner2014regret} showed that the optimal policy cannot be the fixed-arm policy.
One can do much better than always pulling the arm with the maximum average reward.
In this way,
many efforts have been devoted to the MDP-based formulations of the RMAB problems.
However, planning in MDPs is PSPACE-hard \cite{papadimitriou1987complexity}.
To bypass such a challenge,
heuristic algorithms \cite{gittins2011multi} decouple the $N$-dimension problem into $N$ one-dimension subproblems with relaxation techniques.
The WI policy \cite{whittle1988restless} is one of the typical algorithms in this line of research.
It is proved that this policy is asymptotically optimal when the number of arms is sufficiently large \cite{weber1990index} and is also optimal in certain stationary cases \cite{liu2010indexability}.
However, the assumptions that system dynamics are pre-known may not hold in real-world applications,
rendering the above algorithms less applicable.
Therefore, we consider the online RMAB problem with POMDP-based formulation.

In the online RMAB, the agent tracks the optimal selection policy by learning the system dynamics online.
There are two approaches to balance the exploitation-exploration dilemma during the learning process, i.e., the frequency-based method and the Bayes-based method.
The UCB algorithm is one of the frequency-based methods,
which has been adopted in \cite{ortner2014regret, ortner2012regret, wang2020restless}.
Ortner \emph{et al.} \cite{ortner2014regret} proposed a UCRL2 algorithm by using the UCB algorithm to estimate the system dynamics
and solve the estimated MDP with the value iteration method.
By considering the POMDP formulation, Ortner \emph{et al.} \cite{ortner2012regret} further proposed a Colored-UCRL2 algorithm
by grouping the historical observations into finite system states and further assuming that these grouped states follow the MDP dynamics.
Moreover, Wang \emph{et al.} \cite{wang2020restless} devised a restless-UCB algorithm by following the explore-then-commit framework.

Compared with UCB-based algorithms, Bayes-based methods (e.g., the TS algorithm) have a lower computational cost and better performance.
Ouyang \emph{et al.}  \cite{ouyang2017learning} proposed a TSDE algorithm to solve this online MDP problem,
which has a low computational complexity and sublinear regret.
Meanwhile, Tewari \emph{et al.} \cite{akbarzadeh2022learning} and Jung \emph{et al.} \cite{jung2019regret}
proposed two TS-based WI algorithms to handle the POMDP with certain structure information.
However, these algorithms suffer from a high sample complexity in solving the online POMDP.
This is because the accrued samples for estimation are valid
only when an arm is selected twice consecutively.
In this paper, we generalize the WI policy under the POMDP assumption to reduce the sample complexity and
customize the TS algorithm to learn the system dynamics under the federated learning paradigm.
Therefore, the proposed algorithm enjoys both satisfactory performance and high sample efficiency.
\com{Finally, we present Table \ref{tab_ref} to compare the related works and the proposed work in terms of the rewarding process, unknown system dynamics, communication and computation efficiency, and privacy guarantee.}

\begin{table*}[!t]
\renewcommand{\arraystretch}{1.1}
\caption{\com{A Comparison of the Related Works and the Proposed Work.}}
\centering\com{
\begin{tabular}{l||c|c|c|c|c}
\hline\
 & \textbf{MRP} & \textbf{unknown dynamic} & \textbf{commun. efficiency} & \textbf{comput. efficiency} & \textbf{privacy guarantee} \\ [0.5ex]
\hline\hline             
Ref. \cite{shi2021federated}  &   &  &\checkmark  &\checkmark  &\checkmark  \\

Refs. \cite{yang2020federated, xia2020multi, li2022privacy}  &   & &\checkmark  &\checkmark  & \\

Refs. \cite{weber1990index, liu2010indexability}  &\checkmark    & & &\checkmark  &  \\

Refs. \cite{ortner2014regret, ortner2012regret, wang2020restless}  &\checkmark    & & &\checkmark &   \\

Refs. \cite{ouyang2017learning, akbarzadeh2022learning,jung2019regret}   &\checkmark    &\checkmark   & & &\\

The Proposed Work  &\checkmark    &\checkmark  &\checkmark  &\checkmark  &\checkmark  \\
 \hline
 \end{tabular} \label{tab_ref}}
 \end{table*}


\subsection{Organization}
The remainder of this paper is organized as follows.
In Section \ref{SysT}, we present the system model and problem formulation.
In Section \ref{FFwork}, we illustrate the federated online RMAB framework.
The FedTSWI algorithm is proposed in Section \ref{SecFedTSWI} to solve the multi-agent online RMAB problem.
Section \ref{SecCS} considers the online multi-user multi-channel access case to demonstrate the effectiveness of the proposed algorithm.
Simulation results are given in Section \ref{SimSec}, and Section \ref{SecCF} concludes this paper.
\com{For convenience, Table \ref{tab_sym} lists some of the main notations and all of the definitions and propositions in this work.}

\begin{table}[!t]
\renewcommand{\arraystretch}{1.2}
\caption{\com{List of Notations, Definitions, and Propositions.}}
\centering\com{
\begin{tabular}{c|l}
\hline
\textbf{Items} & \textbf{Description} \\ [0.5ex]
\hline\hline                 
$M$  &  The number of agents \\
$N$  &  The number of projects or arms \\
$K$  &  The number of selected arms at each time slot \\
$L$ & The number of episodes at Algorithm 1\\
$T$ & The number of time slots at Algorithm 1\\
$S$  &  The joint state of the $N$ project \\
$\mathcal{S}$ & The set of joint states at each agent \\
$A$  &  The joint action of the $N$ project \\
$\mathcal{A}$ & The set of joint actions at each agent \\
$b$ & The belief state\\
$\theta$  &  The state transition probability or system dynamic \\
$\lambda$ & The Lagrange multiplier or playing charge\\
$W_n$ & The Whittle index of arm $n$\\
\hline
Definition \ref{define0}   &  Policy for the belief MDP problem of \eqref{SysGol} \\
Definition \ref{PasSet}   &  Passive set of the single-armed bandit process \eqref{BelValue01} \\
Definition \ref{deIndex}   &  Indexability of the single-armed bandit process \eqref{BelValue01} \\
Definition \ref{define02}   &  WI policy of the single-agent RMAB problem \eqref{SysGol2}  \\
Definition \ref{define1}   &  The definition of \textit{pseudo} regret in \eqref{Reg01} \\
\hline
Proposition  \ref{propo1}   &  The belief state of the single-armed bandit process \\
Proposition  \ref{propo2}   &  The optimal policy for the single-agent RMAB  \\
Proposition  \ref{propo3}   &  The conditions for driven WI in the case study \\
 \hline
 \end{tabular} \label{tab_sym}}
 \end{table}

\section{System Model and Problem Statement}\label{SysT}
In this section, we first introduce the cooperative resource allocation problem.
Then, we formulate it as a POMDP problem.
To overcome the high computational complexity, we further model this problem as an RMAB problem by decoupling the high-dimension POMDP problem to several one-dimensional belief MDP problems.

\subsection{System Model}\label{SM}
We consider a cooperative resource allocation problem, where $M$ agents collaboratively allocate resources to $N$ projects\footnote{The terminologies of project and arm are interchangeable in this work.} sequentially to maximize their accumulated rewards.
\com{For example, in the multi-user multi-channel access problem, the $M$ agents can be the secondary users and the $N$ projects are the channels.}
Let $m \in \mathcal{M} = \{1, 2, \ldots, M\}$ and $ n \in \mathcal {N} = \{ 1,2,\ldots, N\}$ be the indices of the agent and project, respectively.
Time is slotted into $t=1,2, \ldots, T$.
We assume that each project has several states and actions, denoted by $\mathcal{S}^n$ and $\mathcal{A}^n$.
Hence, the joint state and action for the $N$ projects are denoted by $S = \{s^1, s^2, \ldots, s^N \}$ and $A = \{a^1, a^2, \ldots, a^N\}$, respectively.
Note that each project only has two possible actions, i.e., $\mathcal{A}^n = \{0, 1\}$,
where $a=1$ denotes the active (selected) project, and $a=0$ is the passive project.
In addition, let $R^n(s, a)$ be the reward function of project $n$ in state $s$ when taking action $a$.
\com{For the multi-user multi-channel access case, $R^n(s,a)$ represents the transmission rate when the $n$-th channel's state is $s$ and taking action $a$. }
Here we assume that the reward function is deterministic\footnote{This work can be generalized to the unknown reward case at the cost of some constant factor of the regret.} and the outputs are bounded in $[0,1]$.

Define $\theta_n(s'|s, a)$ as the state transition probability of project $n$ from state $s$ to $s'$ when taking action $a$.
Then, the state transition probability (also known as system dynamic) of the RMAB from state $S$ to state $S'$ when taking action $A$ is given by
\begin{equation}\label{SysPro}
\theta\left(S'|S, A \right) = \prod_{n=1}^{N} \theta_n \left(s'|s, a \right).
\end{equation}
Assume that the $M$ agents have the same system dynamics and are deterministic with unknown parameters, which need to be estimated from the historical observations.
Let $\Omega(\theta)$ be the prior distribution of the system dynamics.
The agent's system dynamics can be viewed as a random realization from the prior distribution at each time slot, i.e., ${\hat{\theta}_t} \sim \Omega(\theta)$. In this cooperative resource allocation problem,  each agent selects $K$ out of $N$ projects to invest at time slot $t$, where $1\leq K < N$.
After that, it receives a reward from the environment denoted by
\begin{equation}\label{ReW}
R_t(S,A) = \sum_{n=1}^{N} R_t^n(s,a)  a^n_t,
\end{equation}
where $\sum_{n\in \mathcal{N}} a^n_t = K$.

\subsection{Partially Observable Markov Decision Process}\label{PF01}
The objective of this cooperative resource allocation problem is to learn the system
dynamics collaboratively while maximizing the expected rewards, which is formulated as
\begin{subequations}\label{SysGol}
\begin{align}
& \underset{}{\max\limits_{\pi}}
& &  \  \frac{1}{T} \sum_{t=1}^{T} \sum_{m=1}^{M}  \mathbb{E} \left[R^{m,\pi}_t(S,A) \right] \qquad  \label{SG_1}\\
& \mathrm{s.t.}
& & R^{m,\pi}_t(S,A) = \sum_{n=1}^{N} R_t^{m, n} (s,a)  a^{m, n}_t,  \label{SG_2}\\
& \mathrm{\quad}
& &  \sum_{n=1}^{N} a^{m,n}_t = K, \ \forall m, t,\label{SG_3}
\end{align}
\end{subequations}
where $ \mathbb{E} \left[ \cdot  \right]$ is the expectation operator
and $R_t^{m, \pi} (S,A)$ is the reward of agent $m$ at the state-action pair $(S, A)$ following the policy $\pi$ at time slot $t$.
In addition, $R_t^{m, n} (s,a)$ and  $a^{m, n}_t$ are the reward and action of agent $m$ and project $n$ at time slot $t$, respectively.

When the system dynamic is known, this problem can be viewed as a standard POMDP problem because the state of the arm is observed only when selected.
Define  $\left(\mathcal{S}, \mathcal{A}, R, \hat{\theta}, \mathcal{H}, O \right)$ as a $6$-tuple  POMDP,
where $\mathcal{S} = \prod_{n \in \mathcal{N}} \mathcal{S}^n$ and $\mathcal{A} = \prod_{n \in \mathcal{N}} \mathcal{A}^n$.
In addition, $\mathcal{H}$ is the set of observations and $\hat{\theta} (S'|S,A)$  is the state transition probability at episode $l$ from state $S$ to state $S'$ when taking action $A$.
Term $O(H|S', A)$ is the conditional observation probability
that an agent in state $S'$ observes $H$ when taking action $A$.

Since the agent cannot directly observe the system state $S$,
it can only make decisions based on observations $\mathcal{H}$.
To realize this, the agent maintains a belief state over the system state set $\mathcal{S}$ and updates it using the current observations.
Let $B(S)$ be the belief state, where $B$ is a probability simplex, representing the probability that the system is in state $S$.
When the agent in belief state $B(S)$ takes the action $A$ and observing $H\in \mathcal{H}$ at time slot $t$,
the belief state $B(S')$   is updated as
\begin{equation}\label{BeUp}
B(S') = \frac{ O(H|S',A) \sum_{S\in\mathcal{S}} \hat{\theta}_l(S'|S,A) B(S)} {\mathrm{Pr}(H|B,A)},
\end{equation}
where
\begin{equation}\label{BeStaTra}
\mathrm{Pr}(H|B,A) = { \sum_{S'\in\mathcal{S}} O(H|S',A) \sum_{S\in\mathcal{S}} \hat{\theta}_l(S'|S,A) B(S)}
\end{equation}
is the belief state transition probability that the agent in a belief state $B$ takes action $A$ and observes $H$.
In addition, the reward function is defined as
\begin{equation}\label{BeReFun}
\bar{{R}}(B,A) = \sum_{S\in \mathcal{S}} B(S) {R}(S,A).
\end{equation}
Hence, Eqs. \eqref{BeUp}-\eqref{BeReFun} constitute a $3$-tuple belief MDP.
Next, we define a policy for this belief MDP.
\begin{definition}\label{define0}
\emph{(Policy)} A policy $\pi$ is a conditional distribution over actions given the belief states, i.e.,
\begin{equation}\label{Pol1}
  \pi(A|B) = \mathbb{P}[A_t = A| B_t = B],
\end{equation}
where $A_t$ and $B_t$ are the action and belief state at time slot $t$, respectively.
\end{definition}

The optimal policy $\pi^{\ast}(A|B)$ for the belief MDP can be obtained by solving the corresponding Bellman equation.
Approximate methods, such as point-based value iteration (PBVI), are widely adopted for solving the belief MDP problem.
However, solving this problem is intractable since the number of system states $\mathcal{S}$ is uncountable.
Typically, the number of system states is $|\mathcal{S}| = \prod_{n\in\mathcal{N}} |\mathcal{S}^n|$ when the true state is fully observed,
which is PSPACE hard in identifying the optimal policy \cite{liu2010indexability}.
To bypass this obstacle, we further model this problem as an RMAB problem.

\subsection{The RMAB Problem}\label{PF02}
Note that the $M$ agents are independent when the system dynamics ${\theta}$ are given. Thus, we can drop the agent's index $m$ and problem \eqref{SysGol} is equivalent to maximizing the expected rewards of each agent, i.e.,
\begin{subequations}\label{SysGol2}
\begin{align}
& \underset{}{\max\limits_{\pi}}
& &  \  \frac{1}{T} \sum_{t=1}^{T}  \sum_{n=1}^N   \mathbb{E} \left[  R_t^{n,\pi} (s,a)  a^{n}_t \right] \quad\quad  \label{SG_11}\\
& \mathrm{s.t.}
& & \sum_{n=1}^{N} a^{n}_t = K, \ \forall t. \label{SG_13}
\end{align}
\end{subequations}
Notice that the per-slot constraint \eqref{SG_13}  can be relaxed by
\begin{equation}\label{ConLonTerm}
\frac{1}{T} \sum_{t=1}^{T} \sum_{n=1}^{N} \left(1-a^{n}_t\right) = N-K.
\end{equation}
The Lagrangian function of problem \eqref{SysGol2} is defined as
\begin{equation}\label{LagRex}\small
\begin{split}
\mathcal{L}(\lambda)=& \sup\limits_{\pi} \   \mathbb{E} \left[ \sum_{n=1}^N \left( \underbrace{\frac{1}{T}  \sum_{t=1}^{T}  \left(R_t^{n,\pi} (s,a)  a^{n}_t + \lambda \left(1-a^{n}_t\right) \right)}\limits_{\mathrm{arm} \ n} \right)  \right] \\
&- \lambda (N-K),
\end{split}
\end{equation}
where $\lambda$ is the Lagrange multiplier, also known as the playing charge \cite{liu2010indexability}.
The above equation can be decoupled into $N$ subproblems.
Each subproblem corresponds to a project (or arm).
For a fixed $\lambda$, the decoupled problem of arm $n$ is given by
\begin{equation}\label{SysGol3}
{\max\limits_{\pi}}
 \quad \frac{1}{T}  \sum_{t=1}^{T}    \mathbb{E} \left[ R_t^{n,\pi} (s,a)  a^{n}_t + \lambda \left(1-a^{n}_t\right) \right].
\end{equation}
In the following,
we refer to each subproblem as a single-armed bandit process and drop the arm's index $n$.

Next, we show the single-agent bandit process can also be formulated as a POMDP.
There are two actions at each arm, i.e., $a\in \{0,1\}$.
The agent observes the true state of the arm when it is selected ($a=1$).
However, it needs to speculate on the true state based on its historical observations when the arm is passive ($a=0$).
This process can be captured by the Bayesian rule: The agent maintains a belief state over the arm's states.
The following proposition shows the belief state update process.
\begin{proposition}\label{propo1}
For the single-armed bandit process, which can also be viewed as a POMDP, the belief state at the next time slot is updated by
\begin{equation}\label{BelUpdate}
b(s')=
\left\{\begin{array}{ll}
\eta{\sum\limits_{s\in \mathcal{S}} \hat{\theta}_l(s'|s, a) b(s)}, &    \mathrm{if} \ a = 0,\\
\hat{\theta}_l(s'|s, a),    &    \mathrm{if} \ a = 1, h =s, \\
\end{array}\right.
\end{equation}
where $\eta =  {1}/{\sum_{s'\in \mathcal{S}}  \sum_{s\in \mathcal{S}} \hat{\theta}_l(s'|s, a) b(s)}$
is a normalizing constant, and $h =s$ is the observed true state.
\end{proposition}
\begin{proof}
Please see  Appendix  \ref{appendix1}.
$\hfill\blacksquare$
\end{proof}

The reward function $\bar{R}(b, a)$  when taking action $a$ at belief state $b$ is given by
\begin{equation}\label{BeReward}
\bar{R}(b, a) = \sum_{s \in \mathcal{S}} b(s) R(s, a).
\end{equation}
The belief state transition probability $\Psi(b'|b, a)$ is defined as
\begin{equation}\label{BeStaTra02}
\Psi(b'|b, a) =
\left\{\begin{array}{cc}
1, &    \mathrm{if} \ a = 0,\\
\sum\limits_{h \in \mathcal{H}} \mathrm{Pr} \left(b' | b, a, h \right) \mathrm{Pr}\left( h|b, a\right),    &    \mathrm{if} \ a = 1, \\
\end{array}\right.
\end{equation}
where
\begin{equation}\label{BeStaTra03}
 \mathrm{Pr} \left(b' | b, a, h \right) =
\left\{\begin{array}{ll}
1,    &    \mathrm{if}\ \left(b, a, h\right) \rightarrow  b', \\
0,    &    \mathrm{otherwise}, \\
\end{array}\right.
\end{equation}
and $\mathrm{Pr}\left( h|b, a\right) = { \sum_{s'\in\mathcal{S}} O(h|s',a) \sum_{s\in\mathcal{S}} \hat{\theta}(s'|s, a) b(s)}$.
Note that $O(h|s',a)$ is well characterized by the state transition probability $P(s'|s,a)$
since $h\in \mathcal{H}$ and $ H= S$ when $a=1$.
\com{For the case of $a=0$, the agent cannot observe anything from the environment when $a=0$. In this situation, the belief state will converge to the stationary distribution of the true system state following a determined trajectory. This trajectory is well defined in  \textbf{Proposition} \ref{propo1}. Therefore, belief state $b'$ is deterministic with probability $1$ when $a=0$.}
Based on \eqref{BelUpdate}, \eqref{BeReward}, and \eqref{BeStaTra02},
we obtain a $4$-tuple $(\mathcal{B}, \mathcal{A}, \bar{R}, \Psi)$ belief MDP.
In other words, the $N$-dimension POMDP of \eqref{SysGol2} is decoupled into $N$ one-dimensional belief MDPs.
According to \cite{gittins2011multi}, these $N$ one-dimensional belief MDPs constitute a RMAB problem for each agent.

To solve this RMAB problem, it is still necessary to establish the relationship between the solutions of the decoupled belief MDPs and problem \eqref{SysGol2}.
In fact, the Lagrange multiplier $\lambda$ is the only parameter that links to different subproblems after decomposition.
This multiplier $\lambda$ can be viewed as the subsidy of passive according to \cite{liu2010indexability}.
While the infimum $\lambda$ required to move a state from the active decision to the passive decision measures the attractiveness of an arm to be selected.
As a result,  the states of each arm can be divided into active and passive sets under a given $\lambda$.
This reveals the role of the solution for the single-agent RMAB problem by activating $K$ arms with the largest values of the infimum $\lambda$.
More details will be discussed in Section \ref{SecTSWI04}.

\subsection{Multi-Agent Online RMAB Problem}\label{PF03}
\com{Existing works often assume that the system dynamics are known prior and the dynamic allocation index methods can be applied to find the asymptotically optimal solution. However, these optimization-based methods are not practical in real-world scenarios where dynamics are unknown or constantly changing. Moreover, solving RMAB problems with unknown dynamics using optimization techniques can be computationally expensive and may not scale well.} In this paper, we consider that the system dynamics are learnt from historical samples among $M$ agents.
Due to this learning process, one cannot solve the single-agent online POMDP problem independently as the agents' samples are coupled.

However, the $N$ one-dimension belief MDPs are independent at each episode by decoupling the $N$-dimension POMDP of \eqref{SysGol2}  into $N$ one-dimension belief MDPs.
From this perspective, problem \eqref{SysGol} can be regarded as $N$ single-armed bandit subproblems in \eqref{SysGol3} given $\lambda$. Each subproblem can be solved together by $M$ agents.
This raises two problems:
i) How to estimate the system dynamics and fuse the information from $M$ agents?
ii) How to solve the single-armed subproblems based on the estimated system dynamics?

\begin{figure}[!t]
\centering
\includegraphics[width=3.3in]{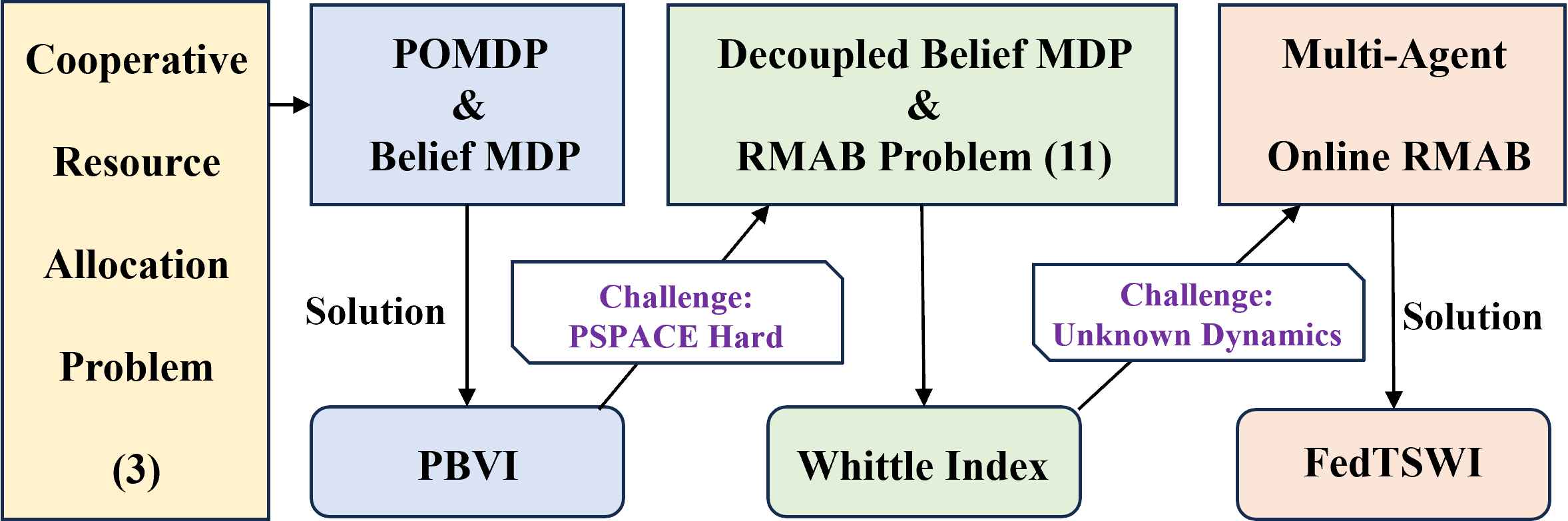}
\caption{\com{The relationship between different problem formulations.}}
\label{PO_Framework}
\end{figure}
In light of this, we regard the RMAB problem as a multi-agent online RMAB problem, where multiple agents collaboratively learn the system dynamics while maximizing their accumulated rewards.
Then, we devise a novel federated online RMAB framework for this problem in Section \ref{FFwork}.
Based on this framework, we propose a FedTSWI algorithm to solve this multi-agent online RMAB problem in Section \ref{SecFedTSWI}.
\com{At last, we present Fig. \ref{PO_Framework} to illustrate the relationship between different problem formulations of this section.}

\section{Federated Online RMAB Framework}\label{FFwork}
\com{We first introduce the federated online RMAB framework for the cooperative resource allocation problem in Section \ref{SecPF01}.
Then, we present the federated TS algorithm to estimate the system dynamics in Section \ref{SecFedTSWI02}.
Based on the estimated dynamics, we apply the WI policy to solve the single-agent subproblem \eqref{SysGol3} in Section \ref{SecTSWI04}.}

\subsection{The Proposed Framework}\label{SecPF01}
Fig. \ref{SysDiag} gives the federated online RMAB framework, which consists of a central server and $M$ agents.
The central server interacts with $M$ agents to learn the system dynamics collaboratively while maximizing their received rewards.
Due to geographic or service heterogeneity, each agent only observes partial information about the environment,
resulting in biased local dynamic estimations.
As a result, the optimal policy can only be established in the central server by collecting all agents' observations.
To address data privacy concerns, we incorporate the FL paradigm into the dynamic estimation process.
Specifically, the samples are kept on the agent side, and only the estimated parameter is uploaded to the central server for model aggregation.
\begin{figure}[!t]
\centering
\includegraphics[width=3.4in]{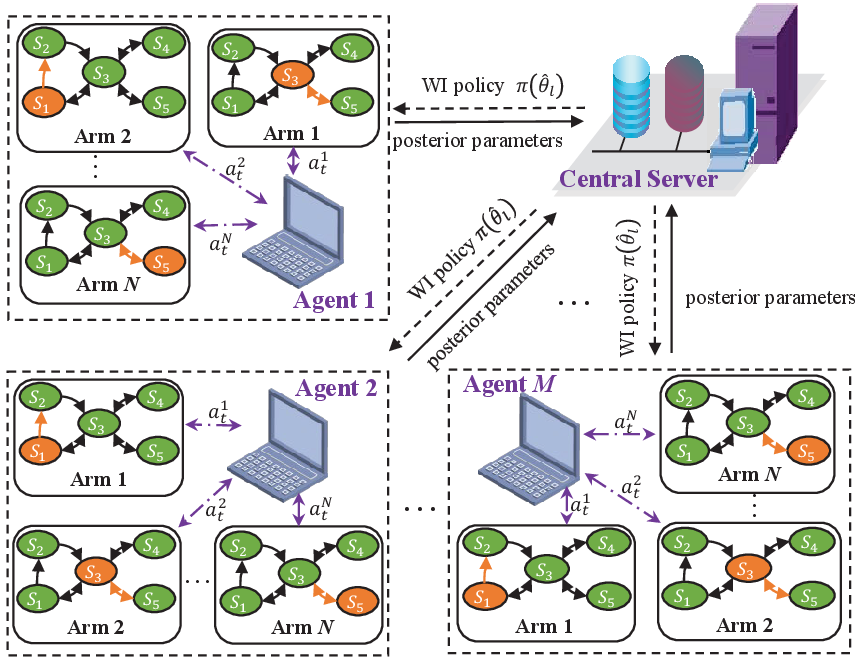}
\caption{An illustration of the federated online RMAB framework.}
\label{SysDiag}
\end{figure}

To reduce the communication overhead,
we divided the learning process into episodes indexed by $l=1,2,\ldots, L$.
The length of each episode is determined by our stopping criteria (see Algorithm \ref{TSWIALG}).
At the beginning of episode $l$, each agent receives the selection policy $\pi$ from the central server.
Then, it executes policy $\pi$ for several time slots and records its observations.
After that, each agent estimates the system dynamics $\hat{\theta}_{m,l}$ using the recorded observations.
At the end of this episode, all agents upload the parameters of the posterior distribution of the system dynamics to the central server for model aggregation.
Based on the aggregated model, the server updates the selection policy $\pi$ and broadcasts it to all agents.
This iteration stops when the time horizon $T$ is exhausted.

\subsection{Federated TS Algorithm for Dynamic Estimation}\label{SecFedTSWI02}
Next, we apply the federated TS algorithm to estimate the system dynamics and aggregate the uploaded information in the central server.
The TS algorithm is a Bayesian method, which maintains a prior distribution on the parameters of the system dynamics $\theta$.
The posterior distribution is updated by multiplying the likelihood function (the observed samples) and the prior distribution.
Let $\Omega_m (\bar{\theta}_{m,l})$ be the prior distribution of the system dynamics of agent $m$ at episode $l$.
The posterior distribution is updated by
\begin{equation}\label{PosUp}
\hat{\Omega}_m (\hat{\theta}_{m, l}) = \frac{O(h| \bar{\theta}_{m, l}) \Omega_m (\bar{\theta}_{m, l})} { {\int_{\theta} O(h| \theta) \Omega_m (\theta) d \theta}},
\end{equation}
where $O(h| \bar{\theta}_{m,l})$ is the likelihood function of the historical observations.
Note that one can approach arbitrary prior and posterior distributions by using the Gibbs sampling or Gaussian process methods \cite{russo2018tutorial}.
\com{To extend this single-agent MAB Bayesian equation \eqref{PosUp} to the multi-agent case, we adopt the method in \cite{lalitha2021bayesian} to aggregate all agents' posterior distribution in the server, i.e.,}
\begin{equation}\label{PosUp02}
\Omega (\hat{\theta}_{l}) =\frac{\exp \left(\sum_{m=1}^{M} \omega_m \log \hat{\Omega}_m (\hat{\theta}_{m, l}) \right)}
{ \int_{\theta} \exp \left(\sum_{m=1}^{M} \omega_m \log \hat{\Omega}_m ({\theta}) \right) d\theta},
\end{equation}
where $\omega_m$ is the sample weight of agent $m$ and $\log (\cdot)$ is the logarithmic function with base $2$.
\com{Weight $\omega$ characterizes the contribution of each agent to the overall posterior distribution.}
We refer to the above process as the federated TS algorithm because each agent only needs to upload the posterior distribution parameters to the central server for model aggregation. This enables the communication efficiency and a privacy guarantee.

The reasons why we adopt the TS algorithm are twofold.
First, it can efficiently balance the exploitation and exploration dilemma,
i.e., exploiting the currently available information to maximize short-term performance or exploring to gather new information that may lead to better long-term performance.
Second, it can reduce the computational complexity and achieve a better performance than the UCB-based algorithms.
The TS algorithm can easily incorporate the problem structures into the prior distribution.
When computing the WI policy, it directly samples a system dynamic from the posterior distribution instead of searching in a confidence region as the UCB-based algorithms.

\subsection{WI Policy for Resource Allocation}\label{SecTSWI04}
\com{The remaining question is how to compute the selection policy $\pi$ based on the estimated $\hat{\theta}$.
In Section 2.3, we show that the Lagrange multiplier $\lambda$ divides the states of each arm into active and passive sets.
The infimum $\lambda$ required to move a state from the active decision to the passive decision measures the attractiveness of an arm to be selected, which is also called the WI.
Therefore, the WI policy for the single-agent RMAB problem is constructed by activating $K$ arms with the largest values of the infimum $\lambda$ or WIs.
However, not every RMAB has a well-defined WI; those that admit a WI policy are called indexable \cite{whittle1988restless}.
In the following, we define the indexability for the single-armed bandit process, i.e., the belief MDP problem given in \eqref{SysGol3}.}

The indexability is established by investigating the structure of the belief MDP.
Let $V_{\lambda}(b)$ be the value function of initial state $b$ with subsidy $\lambda$,
representing the expected total rewards accrued from the single-armed bandit process.
Then, the Bellman optimality equation of the belief MDP is
\begin{equation}\label{BelValue01}\small
V_{\lambda}^{\pi}\left(b\right) = \min\limits_{a\in \{ 0,1\}} \left[ \bar{R}^{\pi}(b, a) +\lambda (1-a) +  \sum\limits_{b' \in \mathcal{B}} \Psi(b'|b,a) V_{\lambda}^{\pi}\left(b'\right) \right]. 
\end{equation}
By substituting  \eqref{BelUpdate},  \eqref{BeReward}, and \eqref{BeStaTra02} into \eqref{BelValue01}, we obtain
\begin{equation}\label{BelValExp}\small
\left\{\begin{array}{ll}
 V_{\lambda}^{\pi}\left(b, a=0\right) = \sum\limits_{s \in \mathcal{S}} b(s) R^{\pi}(s, a) +\lambda +  V_{\lambda}^{\pi}\left(b'\right), \\
V_{\lambda}^{\pi}\left(b, a=1 \right) = \sum\limits_{s \in \mathcal{S}} b(s) R^{\pi}(s, a) +  \sum\limits_{s \in \mathcal{S}} b(s) V_{\lambda}^{\pi}\left(b'\right),  \\
\end{array}\right.
\end{equation}
where $V_{\lambda}^{\pi}\left(b, a=0\right)$ and $V_{\lambda}^{\pi}\left(b, a=1\right)$ are the value functions by taking action $a=0$ and $a=1$ under the policy $\pi$, respectively.
Then, we have the following definitions.
\begin{definition}\label{PasSet}
\emph{(Passive set)} Let $\mathcal{P}(\lambda)$ be the set of belief states $b$ that an arm is optimal to passive with the subsidy $\lambda$.
It satisfies the following conditions:
\begin{equation}\label{P_set}
\begin{split}
   \mathcal{P}(\lambda)
      = \{b: V_{\lambda}^{\pi} \left(b, a = 0\right) \geq V_{\lambda}^{\pi}\left(b, a = 1\right)\}.
\end{split}
\end{equation}
\end{definition}

\begin{definition}\label{deIndex}
\emph{(Indexability)} An arm is indexable if the passive set $\mathcal{P}(\lambda)$ monotonically increases from $\emptyset$ to the whole state space as $\lambda$ increases from $0$ to $+\infty$.
The RMAB problem is indexable if and only if every arm is indexable.
\end{definition}

\com{Based on the indexability, we can derive the WI using \eqref{WI}.}
In particular, the $M$ agents adopt the WI policy to solve the belief MDP problem \eqref{SysGol3} together using the sampled dynamics $\hat{\theta}$.
Let $W_n\left(b, \hat{\theta}_l\right)$ be the WI of arm $n$, which is a function of the belief state $b$ and the system dynamic $\hat{\theta}_l$.
Then, the WI policy can be constructed as
\begin{definition}\label{define02}
\emph{(WI Policy)} If arm $n$ is indexable, its WI is the infimum subsidy $\lambda$ that makes the scheduling decisions (active, passive) equally desirable at state $s$, i.e.,
\begin{equation}\label{WI}
\begin{split}
  W_n\left(b, \hat{\theta}_l\right)
      =  \inf\limits_{\lambda} \left\{\lambda: V_{\lambda}^{\pi} \left(b, a = 0\right) = V_{\lambda}^{\pi}\left(b, a = 1\right)\right\}.
\end{split}
\end{equation}
The WI policy for the single-agent RMAB problem \eqref{SysGol2} is that the agent always selects the $K$ arms with the largest value of WIs at each time slot.
\end{definition}

Therefore, the WI can be obtained by solving \eqref{WI}.
There are many works devoted to deriving the closed-form expression of the WI.
However, the closed-form expression of the WI can only be obtained when the system dynamics have some specific structures.
For example, the linear structure in the AoI-based minimization problem \cite{tong2022age}
and the two-state Markov chain in the multi-channel access problem \cite{liu2010indexability}.
In addition, some works resort to numerical methods to calculate the WI.
The core idea is to find the infimum $\lambda$  in \eqref{WI} by using the binary search algorithm or Newton's method \cite{akbarzadeh2022partially}.
In Section \ref{SecCS}, we provide a closed-form expression of the WI by considering the case study of an online multi-user multi-channel access problem in a cognitive radio network.

\section{The FedTSWI Algorithm}\label{SecFedTSWI}
Based on the federated online RMAB framework, we propose the FedTSWI algorithm to solve the cooperative resource allocation problem \eqref{SysGol}.
Specifically, we present the FedTSWI algorithm in Section \ref{SecFedTSWI04} and derive a regret upper bound for the proposed algorithm in Section \ref{SecFedTSWI05}.

\subsection{Algorithm Description}\label{SecFedTSWI04}
The FedTSWI algorithm is given in Algorithm \ref{TSWIALG}, which consists of two phases: policy improvement and policy evaluation.
We see that the proposed algorithm proceeds in episode $l=1,2,\ldots, L$.
At each episode, the central server performs the policy improvement operations;
Each agent executes the policy evaluation operations for several time slots.
Let $t_l$ be the start time of the $l$-th episode and $T_l = t_{l+1}-t_l$ be the length of the $l$-th episode,
where $t_{l+1}$ is determined by
\begin{equation}\label{stopcriter}
t_{l+1} = \min \left\{t: t > t_l + T_l \ \mathrm{or} \ \mathds{N}_t(s, a) > 2^M \mathds{N}_{t_l}(s, a)  \right\},
\end{equation}
where $\mathds{N}_t(s, a)$ is the number of times that the state-action pair $(s, a)$ has been visited up to time $t$.
The effective length of each episode is dynamically determined by two stopping criteria, i.e.,
i) $ t> t_l + T_l$ and ii) $\mathds{N}_t(s, a) > 2^M \mathds{N}_{t_l}(s, a)$ for some $(s, a)$.
Note that the stopping criteria design is critical for our regret-bound analysis.
\begin{algorithm}[!t]
\caption{The FedTSWI Algorithm}
\label{TSWIALG}
\begin{algorithmic}[1]
\State \textbf{Initialization:} $K, L$, ${M}$, $\mathcal{N}$, $\mathcal{S}$, $\mathcal{A}$, $T$
\State input $b_0$, $t=1$, $t_l=0$, $T_0=1$
\For  {each episode $l=1,2,...,L$, central server}
    \State \emph{$\blacktriangleright$ Policy Improvement (Server)}
    \State Receive posterior parameters from agents
    \State Calculate the posterior distribution using \eqref{PosUp02}
    \State Sample the system dynamics $\hat{\theta}$ from   $\Omega({\hat{\theta}}_{l})$
    \State Compute the WI by solving \eqref{WI} using $\hat{\theta}_l$
    \State Broadcast the WI policy to all agents
    \State \emph{$\blacktriangleright$  Policy Evaluation (Agent)}
    \State Receive the WI policy from the central server
    \While {$t\leq t_l+T_l$ and $\mathds{N}_t(s^n, a^n) \leq 2^M\mathds{N}_{t_l}(s^n, a^n)$ for all $(s^n, a^n)$}
        \State Update the belief state $b$ using \eqref{BelUpdate}
        \State Calculate the WI values of each arm
        \State Select  $K$ arms with the largest WI values
        \State Observe the reward from the selected arms
        \State Record the observations $\mathcal{H}_t$
        \State Update $t = t+1$
    \EndWhile
    \State Update $T_{l} = t-t_l$ and $t_l=t$
    \State Update the posterior distribution $\Omega_m({\hat{\theta}}_{m, l})$ using $\mathcal{H}_t$
    \State Send the parameters of $\Omega_m({\hat{\theta}}_{m,l})$ to the central server
\EndFor
\end{algorithmic}
\end{algorithm}

\emph{Policy Improvement:} The central server first receives the parameters of updated posterior distribution from each agent at the beginning of each episode.
Then, it calculates the merged posterior distribution using \eqref{PosUp02} and samples the system dynamics $\hat{\theta}_l$ from the merged posterior distribution $\Omega({\hat{\theta}}_{l})$.
Based on the sampled dynamics $\hat{\theta}_l$, the central server computes the WI policy and broadcasts it to all agents.
Note that the central server only requires sending the sample dynamics $\hat{\theta}_l$ to the agents when the WI policy has a closed-form expression.

\emph{Policy Evaluation:}
Each agent first receives the WI policy from the server.
Then, they execute the WI policy until triggering the stopping criteria.
During this process, it observes rewards from the selected arms and records them to the historical samples or observations $\mathcal{H}_t$.
This observation cannot only be used to update the belief state $b$ but also be exploited to update the posterior distribution.
Finally, the agent sends the posterior distribution parameters to the central server for model aggregation and policy improvement.

\begin{figure}[!t]
\centering
\includegraphics[width=3.5in]{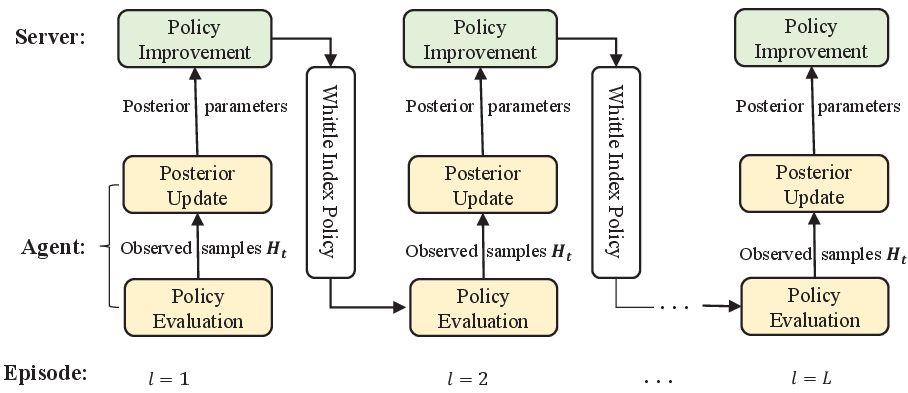}
\caption{The diagram of the FedTSWI algorithm.}
\label{SysDiag02}
\end{figure}

Fig. \ref{SysDiag02} shows the executing process of the FedTSWI algorithm,
which follows the GPI principle, i.e., policy evaluation and improvement.
For the policy evaluation,
each agent keeps tracking the current WI policy and employs the TS algorithm to estimate the system dynamics.
This operation improves the sample efficiency by randomly exploring the system dynamics for heterogeneous agents.
For the policy improvement,
the FedTSWI algorithm ensures the optimal policy is obtained on the server side by carefully balancing the exploitation-exploration dilemma.
This can be achieved by employing the WI policy to solve the POMDP.
The policy evaluation step and the policy improvement step happen iteratively.
The benefits of the proposed algorithm are: i) ensure the communication efficiency and preserve the data privacy;
ii) reduce the sample complexity.

\subsection{A Regret Upper Bound}\label{SecFedTSWI05}
Next, we derive a regret upper bound for the FedTSWI algorithm.
We adopt the definition of the \emph{pseudo} regret to quantify the proposed algorithm, which is given below.
\begin{definition}\label{define1}
\emph{(\emph{Pseudo} Regret)}
The \emph{pseudo} regret is defined  as the expected performance gap between the received rewards accrued from a given policy $\pi$
with unknown dynamics and the rewards obtained following the optimal policy $\pi^{\ast}$ with known dynamics over the time $T=\sum_{l=1}^{L}T_l$. That is
\begin{equation}\label{Reg01}\small
\begin{split}
\mathcal{R}eg(T)
&= \mathbb{E} \left[ \sum_{n=1}^{N} \sum_{l=1}^{L} \sum_{t=1}^{T_l}  \left( R^{n,\pi^{\ast}}_t(b^n,a^n)   - R^{n,\pi}_t(b^n,a^n)  \right)  \right]\\
&= NT \rho (\theta^{\ast}) - \sum_{n=1}^{N} \sum_{l=1}^{L} \sum_{t=1}^{T_l}  \mathbb{E} \left[ R^{n,\pi}_t(b^n, a^n)  \right],
\end{split}
\end{equation}
where $\rho(\theta^{\ast})$ is the optimal average reward obtained by following the optimal policy $\pi^{\ast}$.
\end{definition}
It can be seen that the definition of the pseudo-regret mainly depends on the processes of arm selection and dynamic estimation.
The selection process will induce zero regret when the selection policy is optimal;
Otherwise, it encounters a constant regret, i.e., $\mathcal{O}(\epsilon T)$, where $\epsilon \in [0, 1)$ and $\epsilon=0$ means that the adopted policy is optimal.

However, the estimation process will constitute a main part of the total regrets.
This estimation error can be transformed as the difference between the accumulated rewards of the policy with the sampled system dynamics $\hat{\theta}_l$
and the optimal policy with the actual system dynamics $\theta$.
Our regret analysis has two unique features compared with \cite{ouyang2017learning} and \cite{akbarzadeh2022learning}.
First, we consider the multi-agent setting, which requires carefully designing the stopping criteria to merge all agents' information;
Second, we target the POMDP problem,
where the accrued samples are only valid to learn the system dynamics when a particular arm is selected twice consecutively.
This will result in a constant factor multiplied by the total regret.

Note that \emph{pseudo} regret is the difference between the deterministic and optimal policies when solving problem \eqref{SysGol}.
For specific policies with a constant regret $\mathcal{O}(\epsilon T)$, the regret upper bound of the proposed algorithm is not policy-dependent.
Therefore, we only need to quantify the \emph{pseudo} regret on the learning process by solving a belief MDP based on the sampled dynamics $\hat{\theta}_l$.
In the following, we give some definitions and assumptions to facilitate the analysis.
\begin{definition}\label{defWeakComm}
\emph{(Weakly Communicating MDP)} An MDP is weakly communicating if its states can be partitioned into two subsets.
In the first subset, all states are transient under every stationary policy,
and every two states in the second subset can be reached from each other under some stationary policy.
\end{definition}
\begin{assumption}\label{Assump}
The  belief MDP of problem \eqref{SysGol} is weakly communicating \cite{ouyang2017learning}
and the parameter of the prior distribution of the system dynamics $\Omega (\theta)$ is in a compact set.
Let $SP(\theta)$ be the span of the MDP, i.e., $SP(\theta) = \max_{S\in \mathcal{S}} V(S, \theta)$.
We have $SP(\theta) \leq D$ as the prior distribution of the system dynamics $\Omega (\theta)$ in a compact set.
\end{assumption}

In Algorithm 1, there are $M$ agents collaboratively solving a belief MDP problem at each arm.
Since arms are independent, the total pseudo-regret is obtained by summarizing different arms.
To derive the upper bound to this pseudo-regret, we first need to quantify the number of episodes in Algorithm \ref{TSWIALG}.
Let $L_{T} = \arg \max \{l: t_l\leq T \}$ be the number of episodes before time $T$ for each arm.
Then, we have the following Lemma.
\begin{lemma}\label{lemma1}
The number of episodes at Algorithm \ref{TSWIALG} is \com{upper bounded by}
\begin{equation}\label{Le01}
  L_{T} \leq \sqrt{\frac{4NT|\mathcal{S}^n|}{KM}\log(T+1)},
\end{equation}
where $|\mathcal{S}^n|$ is the number of the states of arm $n$.
\end{lemma}
\begin{proof}
Please see Appendix \ref{appendix3}.
$\hfill\blacksquare$
\end{proof}

Based on Lemma \ref{lemma1}, we can further derive a regret upper bound of the FedTSWI algorithm.
\begin{theorem}\label{Theorem01}
The pseudo regret of the FedTSWI algorithm satisfies the following regret upper bound
\begin{equation}\label{RegBound01}
\begin{split}
\mathcal{R}eg(T) \leq  &(D+1)\sum_{n=1}^{N} \sqrt{\frac{4NT|\mathcal S^n|\log(T+1)}{KM}} \\
&+ 49D \sqrt{\frac{2NT\log(2T)}{KM}}  \sum_{n=1}^{N}|\mathcal S^n| + \mathcal{O}(\epsilon T).
\end{split}
\end{equation}
\end{theorem}
\begin{proof}
Please see Appendix \ref{appendix4}.
$\hfill\blacksquare$

\end{proof}

\begin{remark}\label{remark1}
It is seen that the pseudo-regret contains three terms.
The first two terms are the regrets incurred by the dynamic estimation error.
The third term is due to the suboptimal policy, and it will trend to zero when the WI policy is asymptotically optimal.
In this way, the regret upper bound increases sub-linearly with time horizon $T$,
i.e., $\mathcal{R}{eg}(T) = \mathcal{O}(\sqrt{T\log T})$.
This indicates that the proposed algorithm converges to the optimal policy as the per-round regret approaches $0$ when $T$ is sufficiently large.
\end{remark}
\begin{remark}\label{remark2}
The regret upper bound also depends on the system model parameters $|{S^n}|$, $N$, $K$, and $M$.
More importantly, it shrinks sub-linearly when the number of selected arms $K$ and agents $M$ increases.
Therefore, the FedTSWI algorithm has a lower sample complexity in solving the online POMDP compared with the single-agent case.
\end{remark}

\section{Case Study: Online Multi-User Multi-Channel Access}\label{SecCS}
We demonstrate the effectiveness of  the proposed algorithm by studying an online multi-channel access problem in cognitive radio networks,
where multiple users collaboratively learn the channel state transition probabilities while maximizing their average transmit rate.
We present the closed-form expression of the WI in Section \ref{SecCFWI} and apply the federated TS algorithm to learn the system dynamics using the Beta-Bernoulli distribution in Section \ref{SecTSBB}.

\subsection{Dynamic Multi-Channel Access}
In cognitive radio networks, multiple secondary users need to perform spectrum sensing before accessing the licensed spectrum in case of interfering with the primary users \cite{tong2018cooperative}.
We assume that each secondary user can listen to $N$ independent Gilbert-Elliot channels (arms) with transmission rate $\nu_n, \forall n \in \mathcal{N}$.
Without loss of generality, we normalize the maximum data rate to $1$, i.e., $\max_n \nu_n = 1$.
\com{The sensing result reflects the channel status. We consider that each channel only has two states, i.e., $0$ for the ``bad" state and $1$ for the ``good" state.}
The Markov state transition matrix can be characterized by the probability vector $[\theta^n_{01}, \theta^n_{11}]$ as shown in Fig. \ref{MarStaDiag}.
At each time slot, each secondary user selects $K$ out of the $N$ channel to perform spectrum sensing.
If the state of the sensed channel is good, it transmits and receives $\nu_n$ reward;
Otherwise, it receives zero reward.
Hence, the reward observed in time slot $t$ is given by
\begin{equation}\label{CS01}
R_t (s^n_t a^n_t) = \sum_{n=1}^{N} \nu_n s^n_t a^n_t,
\end{equation}
where $s^n_t \in \{ 0 ,1\}$ and $a^n_t \in \{ 0, 1\}$ are the state and action of channel $n$ at time slot $t$, respectively.
\begin{figure}[!t]
\centering
\includegraphics[width=1.6in]{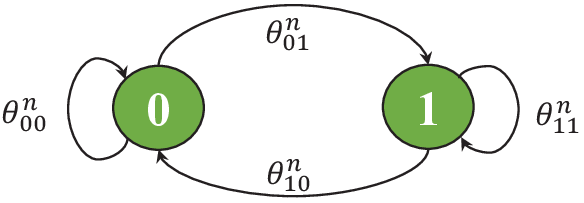}
\caption{The two-state Markov chain of arm $n$.}
\label{MarStaDiag}
\end{figure}

In fact, the channel state can be observed only when sensed.
For the non-selected channels, users are required to deduce their true states from historical decisions and observations.
According to Eq. \eqref{BelUpdate}, given the sensing action $a^n_t$ and the observation,
the belief state in time slot $t+1$ is
\begin{equation}\label{CS02}
b^n(s')=
\left\{\begin{array}{ll}
\theta^n_{11}, &    \mathrm{if} \ a^n_t = 1, s^n_t =1,\\
\theta^n_{01}, &    \mathrm{if} \ a^n_t = 1, s^n_t =0,\\
b^n(s)\theta^n_{11} +(1-b^n(s))\theta^n_{01}, &    \mathrm{if} \ a^n_t =0.
\end{array}\right.
\end{equation}
Let $b^n_0 = \theta^n_{01}/(\theta^n_{01}+1 - \theta^n_{11})$ be the initial state of arm $n$.
The reward function $\bar{R}(b, a)$  when taking action $a$ at belief state $b$ is given by
\begin{equation}\label{BeReward05}
\bar{R}(b, a) = \sum_{s \in \mathcal{S}} b(s) R(s, a).
\end{equation}
The system's goal is to maximize the expected reward by designing a sensing policy $\pi$ to sequentially select channels at each time slot, i.e.,
\begin{align}\label{SysGol05}
& \underset{}{\max\limits_{\pi}}
& &  \  \lim_{T\rightarrow \infty} \frac{1}{T}\sum_{t=1}^{T} \bar{R}_t(b, a) a^{\pi}_t,
\end{align}
where the time horizon $T = \sum_{l=1}^{L} T_l$.

\subsection{Closed-Form Whittle Index}\label{SecCFWI}
As mentioned in Section \ref{PF02}, problem \eqref{SysGol05} can be formulated as a belief MDP.
Its Bellman optimality equation is
\begin{equation}\label{BelValue}\small
V_{\lambda}^{\pi}\left(b\right) = \min\limits_{a\in \{ 0,1\}} \left[ \bar{R}^{\pi}(b, a) +\lambda (1-a) +  \sum\limits_{b' \in \mathcal{B}} \Psi(b'|b,a) V_{\lambda}^{\pi}\left(b'\right) \right]. 
\end{equation}
By substituting  \eqref{BelUpdate},  \eqref{BeReward}, and \eqref{BeStaTra02} into \eqref{BelValue}, we obtain
\begin{equation}\label{BelValExp}\small
\left\{\begin{array}{ll}
 V_{\lambda}^{\pi}\left(b, a=0\right) = \sum\limits_{s \in \mathcal{S}} b(s) R^{\pi}(s, a) +\lambda +  V_{\lambda}^{\pi}\left(b'\right), \\
V_{\lambda}^{\pi}\left(b, a=1 \right) = \sum\limits_{s \in \mathcal{S}} b(s) R^{\pi}(s, a) +  \sum\limits_{s \in \mathcal{S}} b(s) V_{\lambda}^{\pi}\left(b'\right),  \\
\end{array}\right.
\end{equation}
where $V_{\lambda}^{\pi}\left(b, a=0\right)$ and $V_{\lambda}^{\pi}\left(b, a=1\right)$ are the value functions by taking action $a=0$ and $a=1$ under the policy $\pi$, respectively.
Next, we show the single-agent RMAB is indexable by exploiting the structure of the Bellman optimality equation \eqref{BelValExp}.
According to Definition \ref{deIndex}, we have the following proposition and corollary.
\begin{proposition}\label{propo2}
The optimal policy for the single-agent RMAB with subsidy $\lambda$ is a threshold policy.
There exists a belief state $b^{\ast}(\lambda)$ such that $V_{\lambda}^{\ast} \left(b^{\ast}(\lambda), a = 0\right) = V_{\lambda}^{\ast}\left(b^{\ast}(\lambda), a = 1\right)$, which divides the belief space into two regions corresponding to active and passive actions.
\end{proposition}
\begin{proof}
Please see Appendix \ref{appendix2}.
$\hfill\blacksquare$
\end{proof}

\begin{corollary}\label{corol1}
The single-agent RMAB is indexable and its indexability is independent of the system dynamics.
\end{corollary}
\begin{proof}
Please see Appendix \ref{appendix5}.
$\hfill\blacksquare$
\end{proof}

Based on the above proposition and corollary, we now present the WI policy.
Let $\mathds{T}^j (b(t)) \triangleq \mathrm{Pr} [s(t+j) =1 | b(t)] $ be the $j$-step belief update when the arm is unobserved for $j$ consecutive rounds. \com{Then, we give the following two equations in Proposition \ref{propo3} to derive the WI policy.
\begin{proposition}\label{propo3}
By exploiting the definition of belief state in (27) and the eigendecomposition of the transition matrix $[\theta_{00}, \theta_{01}; \theta_{10}, \theta_{11}]$, we have
\begin{equation}\label{CS00A}
\mathds{T}^j (b) =   \frac{\theta_{01} - (\theta_{11}-\theta_{01})^j (\theta_{01} - (1+\theta_{01} -\theta_{11}b))}{1+\theta_{01}-\theta_{11}},
\end{equation}
and
\begin{equation}\label{CS00B}
\min \{ \theta_{01}, \theta_{11}\} \leq \mathds{T}^j (b) \leq \max \{ \theta_{01}, \theta_{11}\}, \ \forall j\geq 1.
\end{equation}
\end{proposition}
\begin{proof}
Please see Appendix \ref{appendix6}.
$\hfill\blacksquare$
\end{proof}}

In addition, we define $\mathds{L} (b, b')$ as the minimum amount of time required for a passive arm to transit across state $b'$ starting from $b$,
i.e., $\mathds{L} (b, b') \triangleq \min \{ j:\mathds{T}^j (b)  > b' \}$.
The Whittle index in the case of $\theta_{11}\geq \theta_{01}$ is given in  Eq. \eqref{WI01}, as shown at the top of this page.
\newcounter{MYtempeqncnt3}
\setcounter{MYtempeqncnt3}{\value{equation}}
\setcounter{equation}{33}
\begin{figure*}[!t]
\normalsize
\begin{equation}\label{WI01}
\renewcommand{\arraystretch}{1.6}
W_n(b)=
\left\{\begin{array}{ll}
b \nu_n ,
&    \mathrm{if} \ b\leq \theta_{01} \ \mathrm{or} \ b \geq \theta_{11}, \\
\frac{\left(b-\mathds{T}(b)\right)\left(\mathds{L}(\theta_{01}, b)+1 \right) + \mathds{T}^{\mathds{L}(\theta_{01},b)} (\theta_{01})  }
{1-\theta_{11} + \left( b -  \mathds{T}(b)\right) \mathds{L}(\theta_{01}, b)  +  \mathds{T}^{\mathds{L}(\theta_{01},b)} (\theta_{01})} \nu_n,
&    \mathrm{if}  \ \theta_{01} < b< b_0, \\
\frac{b}{1-\theta_{11} +b} \nu_n,
&    \mathrm{if}  \ b_0  \leq b< \theta_{11}.\\
\end{array}\right.
\end{equation}
\setcounter{equation}{\value{equation}}
\hrulefill
\end{figure*}
For the case of $\theta_{11}< \theta_{01}$, the Whittle index is given by \cite{liu2010indexability}
\begin{equation}\label{WI02}\small
\renewcommand{\arraystretch}{1.6}
W_n(b)=
\left\{\begin{array}{ll}
b \nu_n ,
&    \mathrm{if} \ b\leq \theta_{11} \ \mathrm{or} \ b \geq \theta_{01}, \\
\frac{b + \theta_{01} - \mathds{T}(b)}{1+ \theta_{01} - \mathds{T}(\theta_{11}) + \mathds{T}(b) -b} \nu_n,
&    \mathrm{if}  \ \theta_{11} < b< b_0, \\
\frac{\theta_{01}} {1+ \theta_{01} - \mathds{T} (\theta_{11})} \nu_n,
&    \mathrm{if}  \ b_0  \leq b<  \mathds{T} (\theta_{11}),\\
\frac{\theta_{01}} {1+ \theta_{01} - b} \nu_n,
&    \mathrm{if}  \ \mathds{T} (\theta_{11}) \leq b< \theta_{01}.
\end{array}\right.
\end{equation}

\subsection{Dynamics Estimation in Beta-Bernoulli Distribution}\label{SecTSBB}
Although we have obtained the closed-form expression of the Whittle index,
the state transition probability\footnote{Note that $\{ \theta_{01}, \theta_{11}\}$ is sampled from the posterior distribution $\Omega (\theta)$.} $\{ \theta_{01}, \theta_{11}\}$ is still unknown and needs to be estimated from the historical samples.
\com{We adopt the FL framework for parameter aggregation, which not only considers the communication efficiency issue but also addresses the data privacy issue (e.g., location privacy and regulatory requirement).}
According to Algorithm \ref{TSWIALG}, the TS algorithm is used to estimate the dynamics by integrating the samples from different agents.
In fact, the prior and posterior distributions in this dynamic channel access case can be viewed as the Beta-Bernoulli distribution.
Specifically, the likelihood function is the Bernoulli distribution since the sensed channel state is either good or bad.
Hence, the prior distribution of $\theta$ can be viewed as a Beta($\alpha, \beta$) distribution.
Its probability density function (pdf) is given by
\begin{equation}\label{BB02}
\begin{split}
\Omega^{\mathrm{Beta}}_{\alpha,\beta}(\theta)&=\frac{\theta^{\alpha-1}(1-\theta)^{\beta-1}}{B(\alpha,\beta)}
\propto \theta^{\alpha-1}(1-\theta)^{\beta-1},
\end{split}
\end{equation}
where $ B(\alpha,\beta)=\frac{\Gamma(\alpha)\Gamma(\beta)}{\Gamma(\alpha+\beta)}$ is the beta coefficient.
In addition, the pdf of the binomial $(k,n)$ distribution (i.e., the total $n$ Bernoulli trails with $k$ positive results) is given by
\begin{equation}\label{BB03}
\begin{split}
O^{\mathrm{Bion}}_{k,n}(\theta)&= \binom{n}{k}\theta^{k}(1-\theta)^{n-k}
\propto \theta^{k}(1-\theta)^{n-k},
\end{split}
\end{equation}
where $\binom{n}{k} = n!/(k!(n-1)!)$ is the binomial coefficient.
According to \eqref{PosUp02}, the posterior distribution of $\theta$ at the central server can be updated by
\begin{equation}\label{BB04}
\begin{split}
\Omega (\hat{\theta}_{l}) &=\frac{\exp \left(\sum_{m=1}^{M} \omega_m \log \hat{\Omega} (\hat{\theta}_{m, l}) \right)}
{ \int_{\theta} \exp \left(\sum_{m=1}^{M} \omega_m \log \hat{\Omega}_m ({\theta}) \right) d\theta}\\
&=\frac{ \Omega(\bar{\theta}_{ l}) \prod_{m=1}^{M}  \left(  O(h| \bar{\theta}_{m, l}) \right)^{\omega_m} }
{ \int_{\theta} \exp \left(\sum_{m=1}^{M} \omega_m \log \hat{\Omega}_m ({\theta}) \right) d\theta}\\
&\propto \Omega(\bar{\theta}_{ l})   \prod_{m=1}^{M}  \left( O(h| \bar{\theta}_{m, l}) \right)^{\omega_m} \\
&= \left(\bar{\theta}_{l}\right)^{\alpha -1 +\sum_{m=1}^{M} \mathds{N}^m_{t_l} (s=1,a) }  \\
& \quad \quad \quad \quad  \quad \times \left(1-\bar{\theta}_{l}\right)^{\beta -1 +\sum_{m=1}^{M} \mathds{N}^m_{t_l} (s=0,a) },
\end{split}
\end{equation}
where $\mathds{N}^m_{t_l} (s,a)$ is the number of times that the state-action pair $(s,a)$ of agent $m$ has been visited and $\mathds{N}^m_{t_l} (s,a) = \mathds{N}^m_{t_l} (s=1,a) + \mathds{N}^m_{t_l} (s=0,a)$.
Therefore, at the end of each episode, the agent only needs to send  $\mathds{N}^m_{t_l} (s, a)$ to the central server for posterior update.

Finally, the online multi-user multi-channel access problem can be solved by running the Algorithm \ref{TSWIALG} using the closed-form WIs
and the Beta-Bernoulli distribution for dynamics estimation.
Each agent is only required to upload $\mathds{N}_{t_l} (s,a)$ to the central server at each episode.
Therefore, the proposed algorithm can achieve a low implementation complexity and communication overhead.
This will further be validated by the numerical results in the following section.

\section{Simulations}\label{SimSec}
This section presents simulations to evaluate the FedTSWI algorithm with different configurations under the online multi-user multi-channel access case.
Each channel follows the two-state Markov model as shown in Fig. \ref{MarStaDiag},
which $0$ stands for ``bad" channel state and $1$ for ``good" channel state.
The state transition probability of each arm is characterized by the $\theta^n_{01}$ and $\theta^n_{11}$.
Specifically, the state transition probabilities of different agents are set to $\{ (\theta^n_{01}, \theta^n_{11})\}_{n\in \mathcal{N}} = \{ (0.20, 0.80), (0.89, 0.17), (0.1, 0.9), (0.9, 0.16)\}$.
\com{In addition, the normalized transmission rates of each arm are $\nu_{n\in \mathcal{N}} = \{ 0.4, 0.9, 0.7, 0.6\}$.
Thus, the rewards of different arms are $\{ 0, 0, 0, 0\}$ when the channels are in the busy state,
and are $\{ 0.4, 0.9, 0.7, 0.6\}$ when the channels are in an idle state.}
For convenience, the initial belief state\footnote{Note that the initial belief state can be an arbitrary value in $[0,1]$.} is set to the stationary distribution of the Markov chain, i.e., $b^n_0 = \theta^n_{01}/(\theta^n_{01} + 1-\theta^n_{11})$.
All results are obtained from $10^4$ Monte Carlo trials.

\begin{figure}[!t]
\centering
\includegraphics[width=2.8in]{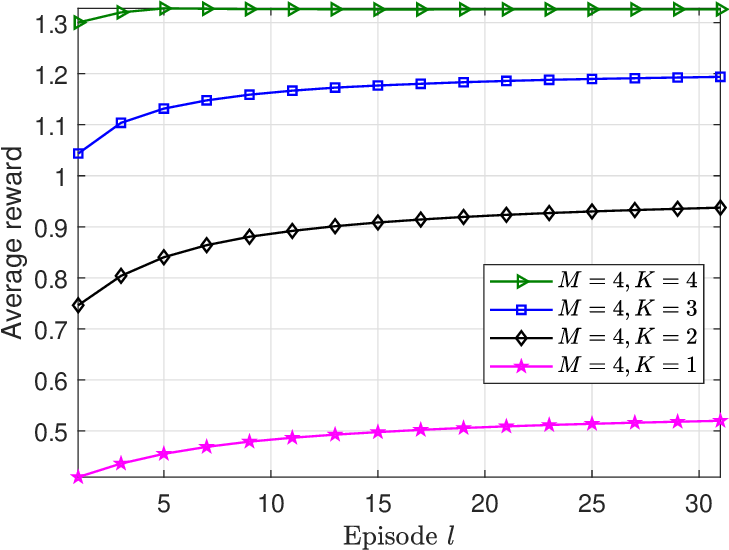}
\caption{The average rewards of the FedTSWI algorithm under different number of selected arms $K=\{1, 2, 3, 4\}.$}
\label{MultiAlgReg2Slot}
\end{figure}
We first consider the case that the number of agents and arms are $M=4$ and $N=4$, respectively.
While the number of selected arms at each time slot changes from  $1$ to $4$.
Fig. \ref{MultiAlgReg2Slot} depicts the average rewards of the FedTSWI algorithm under different numbers of selected arms.
We see that the average rewards increase with the number of episodes and the number of selected arms.
In addition, the convergence rate of the dynamic estimation process also increases with the number of selected arms.
Especially when $K=4$, the proposed algorithm only takes one or two episodes to converge.

\begin{figure}[!t]
\centering
\includegraphics[width=2.8in]{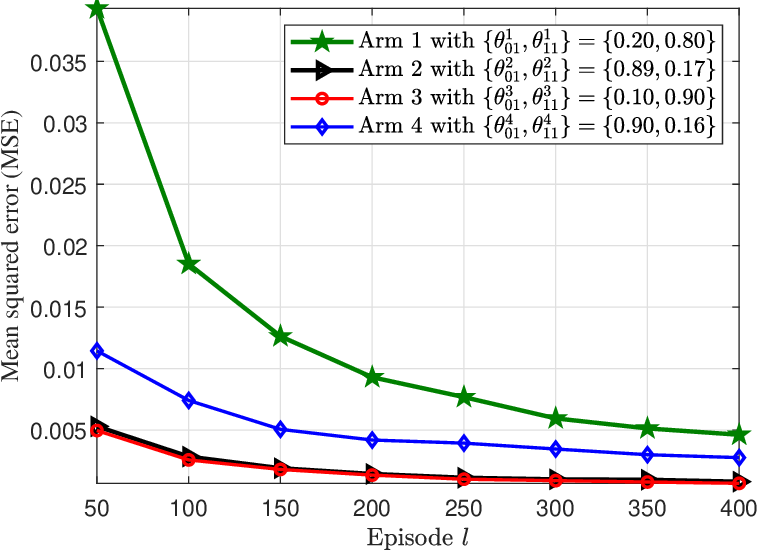}
\caption{\com{The MSEs of of the estimated dynamics $\hat{\theta}$ for different arms by running the FedTSWI algorithm when $M=4$, $N=4$, and $K=2$.}}
\label{DynamicMSE01}
\end{figure}
\com{To further evaluate the convergence rate of the proposed algorithm, we present the mean squared errors (MSEs) of the estimated dynamics $\theta$ for different arms by running the FedTSWI algorithm when $M=4$, $N=4$, and $K=2$, as shown in Fig. \ref{DynamicMSE01}.
The MSE of the estimated dynamics at each arm is defined on the difference between the true dynamics $\{ (\theta^n_{01}, \theta^n_{11})\}_{n\in \mathcal{N}} = \{ (0.20, 0.80), (0.89, 0.17), (0.1, 0.9), (0.9, 0.16)\}$ and the estimated dynamics $\{ (\hat{\theta}^n_{01}, \hat{\theta}^n_{11})\}_{n\in \mathcal{N}}$. In Fig. \ref{DynamicMSE01}, we plot the average MSE of the two states ${ (\theta^n_{01}, \theta^n_{11})}$ for each arm. We observe that the MSEs of the four arms decrease with the number of episodes. In particular, arms 2 and 3 exhibit the lowest MSEs among the four arms, as they have been selected more frequently. This is attributed to the fact that arms 2 and 3 produce the highest rewards among the arm combinations (i.e., $\nu_{n=\{2,3\}} = \{ 0.9, 0.7\}$). Figs. \ref{MultiAlgReg2Slot} and \ref{DynamicMSE01} demonstrate that the convergence rate of the proposed algorithm depends on the number of selected arms $K$ and episodes $L$.
}

\begin{figure}[!t]
\centering
\includegraphics[width=2.8in]{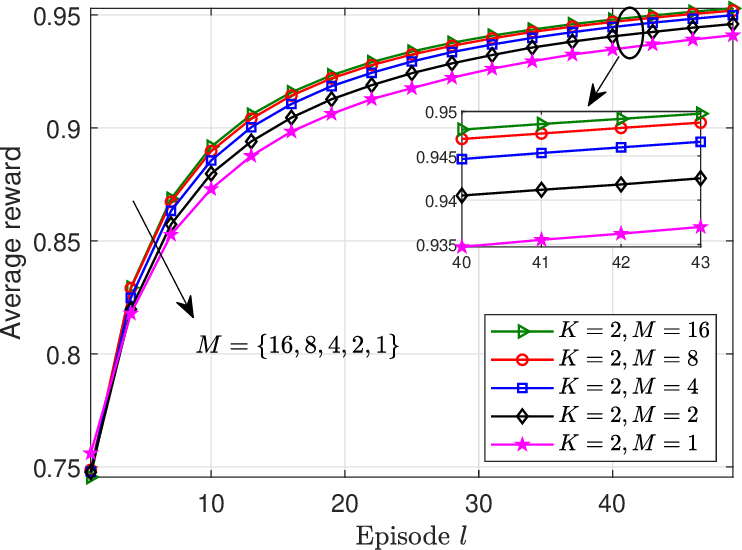}
\caption{The average rewards of the FedTSWI algorithm under a different number of agents $M=\{1, 2, 4, 8, 16\}$.}
\label{MultiAlgRew2Agent}
\end{figure}
Fig. \ref{MultiAlgRew2Agent} shows the average reward of the FedTSWI algorithm under different numbers of agents $M=\{1, 2, 4, 8, 16\}$, where the number of selected arms is $K=2$.
We see that the average rewards of different cases increase with the number of episodes.
The performance of the FedTSWI algorithm improves significantly when the number of agents increases from $M=1$ to $M=4$.
This indicates that the federated online RMAB framework works well,
where the multi-agent setting is more sample efficient compared to the single-agent case (i.e., $M=1$).
However, the performance gain becomes less significant when the number of agents is over $8$.
This is because the two-state Markov model only requires a few samples to estimate the system dynamics.

Next, we compare the FedTSWI algorithm with some baseline algorithms, including the optimal value, the WI policy, the myopic policy, the FedTS-myopic algorithm, the random selection policy, and the fixed selection policy.
The optimal value is obtained by solving the POMDP problem defined in Section \ref{PF02} using the point-based value iteration method with known system dynamics.
At each time slot $t$,
the WI policy is performed by activating $K$ out of $N$ arms with the maximum value of the WIs, where the system dynamics are known prior.
\com{The myopic policy selects $K$ out of $N$ arms with the maximum value of $\bar{R}(b,a)$ in \eqref{BeReward}.
To implement the FedTS-myopic algorithm, we only need to change the line "Select $K$ arms with the largest WI values" in Algorithm 1 to "Select $K$ arms with the largest $\bar{R}(b, a)$".}
The random selection policy randomly chooses $K$ out of $N$ arms from arm set $\mathcal{N}$ at each time slot.
In addition, the fixed selection policy always selects the $K$ arms with the maximum values of the initial belief state $b^n_0$.
Note that the random and fixed selection policies do not require information on the system dynamics.

\begin{figure}[!t]
\centering
\subfloat[Stochastically identical arm]{\label{Identical}  
\includegraphics[width=0.45 \columnwidth]{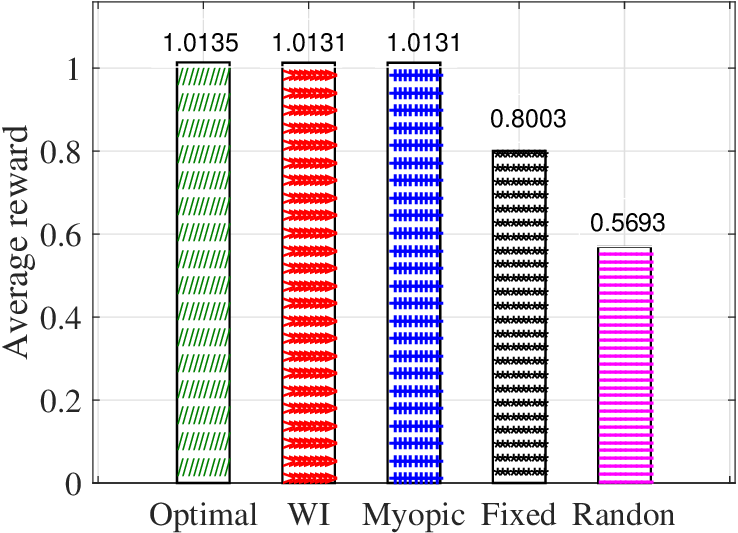}}  
\hfil
 \subfloat[Stochastically non-identical arm] {\label{NonIden}
\includegraphics[width=0.45 \columnwidth]{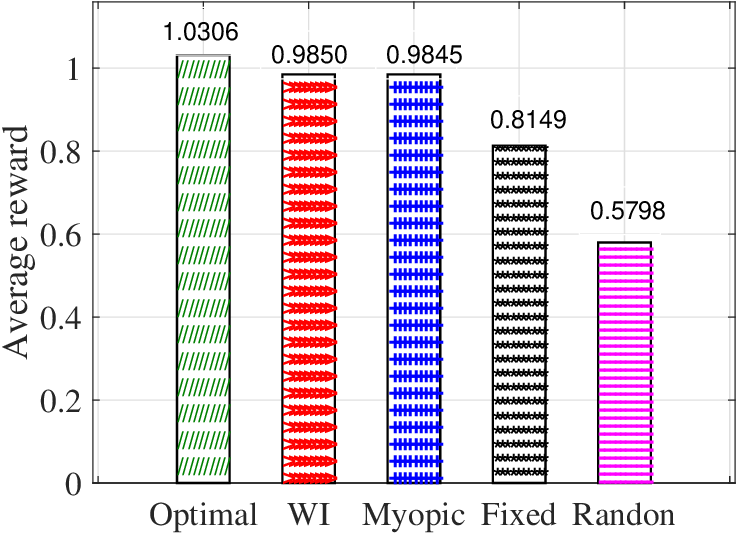}}  
\caption{The average rewards of different baseline algorithms with known dynamics under the stochastically and non-identical arm settings.}  
\label{OfflinePolicy} 
\end{figure}
Fig. \ref{OfflinePolicy} shows the average rewards of the optimal value, the WI policy, the myopic policy, the fixed selection policy, and the random selection policy under the settings of $M=4$ and $K=2$, where the system dynamics are given.
The setting of the identical stochastic arms is that each arm has the same dynamic. i.e., $(\theta^n_{01}, \theta^n_{11}) = (0.20, 0.80), \ \forall n \in \mathcal{N}$.
From Fig. \ref{Identical}, we see that the WI, myopic, and optimal policies have the same performance.
This matches the result in \cite{liu2010indexability}, where the WI and myopic policies are optimal in the stochastically identical arm setting.
However,  the WI and myopic policies are asymptotically optimal under the stochastically non-identical arm setting, as shown in Fig. \ref{NonIden}.
In addition, the fixed selection policy is worse than the WI and myopic policies in both cases.
This is a good example of how one can do much better than always pulling the best one's weight in the RMAB problem.

\begin{figure}[!t]
\centering
\includegraphics[width=2.8in]{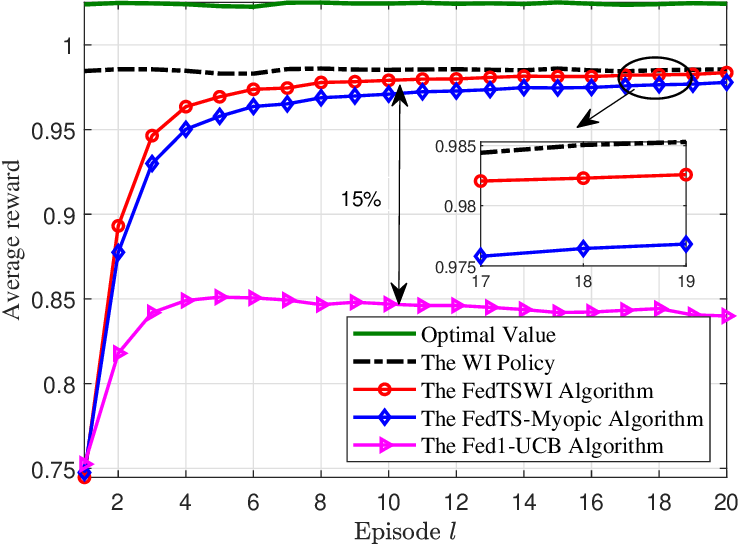}
\caption{The average rewards of the optimal policy,  WI policy,  FedTSWI algorithm,  FedTS-Myopic algorithm, and \com{the Fed1-UCB algorithm.}}
\label{MultiAlgRew}
\end{figure}
\setlength{\textfloatsep}{4pt}
\com{Fig. \ref{MultiAlgRew} compares the performance of the FedTSWI algorithm, optimal value, WI policy, FedTS-Myopic algorithm, and the Fed1-UCB algorithm under the stochastically non-identical arm setting, where $M=4$ and $K=2$.
The Fed1-UCB algorithm was proposed in \cite{shi2021federated} for the multi-agent stochastic MAB problem with the exact model (i.e., the agent and server have the same arm distribution). 
The goal of the server and agent is to find the best arms $2$ and $3$ using the UCB algorithm.}
It is seen that the performance of the FedTSWI algorithm is very close to the WI policy.
This indicates that the FedTSWI algorithm can successfully estimate the system dynamic while following the target policy, i.e., the WI policy.
In addition, the average reward of the FedTSWI algorithm is better than the FedTS-Myopic algorithm because the myopic policy ignores the impact of the current action on the future reward\footnote{\com{Note that both the WI policy and the Myopic policy are established on the belief state $b$ rather than the channel state $s=\{0, 1\}$.}}, focusing solely on maximizing the expected immediate reward.
There is a performance gap between the FedTSWI algorithm and the optimal value.
This aligns with the results of Fig. \ref{NonIden} and Theorem \ref{Theorem01} with $\epsilon = 4.56\times 10^{-5}$.
\com{Moreover, the Fed1-UCB algorithm accounts for the worst performance because it does not consider the Markov reward process at each arm. The proposed algorithm improves the average reward by $15\%$ compared with the Fed1-UCB algorithm.}

\begin{figure}[!t]
\centering
\includegraphics[width=2.8in]{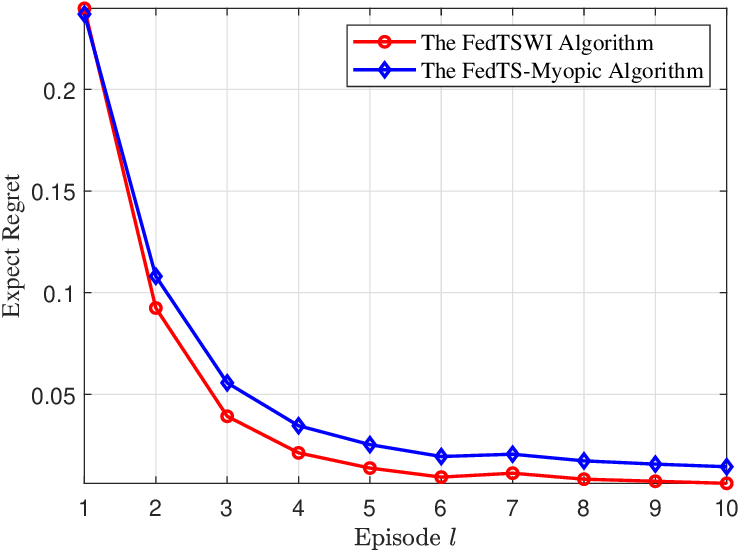}
\caption{\com{The expect regrets of the FedTSWI algorithm and FedTS-Myopic algorithm, where $M=4$ and $K=2$.}}
\label{OnlineRegret}
\end{figure}
\com{Fig. \ref{OnlineRegret} shows the expected regrets of the FedTSWI algorithm and the FedTS-Myopic algorithm, where $M=4$ and $K=2$. The regret is defined as the performance gap between the received rewards accrued from the given algorithms and the rewards obtained by following the Whittle policy with known system dynamics. We observe that the expected regrets of both algorithms decrease with the number of episodes. However, their expected regrets will not converge to zero, as the system's dynamic estimation process and the selection policy introduce some bias.}

\begin{figure}[!t]
\centering
\includegraphics[width=2.8in]{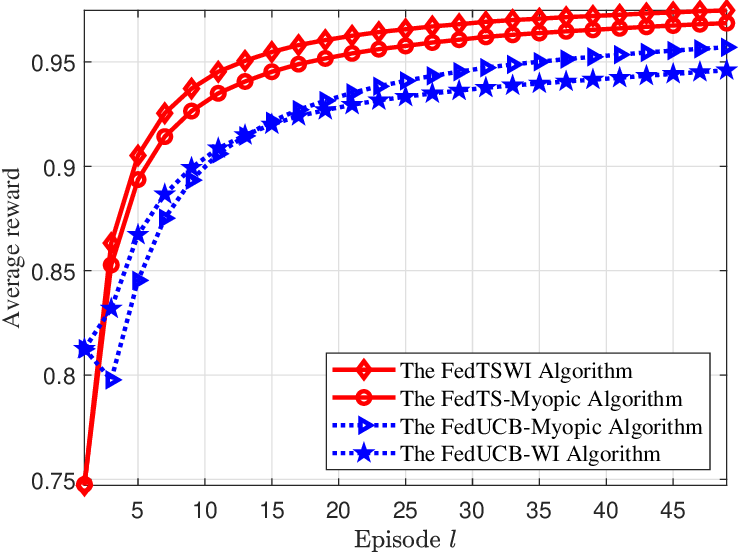}
\caption{The average rewards of the FedTSWI algorithm, FedTS-Myopic algorithm,  FedUCB-WI algorithm, and the FedUCB-Myopic algorithm.}
\label{MultiAlgUCB2TS}
\end{figure}
Finally, we evaluate the dynamic estimation process by considering the federated TS algorithm and the federated UCB algorithm.
Similar to the UCRL2 algorithm in \cite{auer2008near}, we define the UCB index by
$\Pi^{\mathrm{ucb}}_{n} (l) = \hat{\theta}_{l,n} + \sqrt{{\log(|\mathcal{A}| |\mathcal{S}|l)}/{\Upsilon_{n,l}}}$,
where $\Upsilon_{n,l}$ is the number of times that arm $n$ has been selected up to episode $l$.
Then, the  UCB index is used to constitute the FedUCB-WI algorithm and the FedUCB-Myopic algorithm with the WI and myopic policies, respectively.
Fig. \ref{MultiAlgUCB2TS} presents the average rewards of the FedTSWI algorithm, FedTS-Myopic algorithm,  FedUCB-WI algorithm, and the FedUCB-Myopic algorithm under the stochastically non-identical arm setting, where $M=4$ and $K=2$.
We see that the TS-based algorithms outperform the UCB-based algorithms.
Moreover, the UCB-based algorithms suffer from a slow convergence rate.
This is because the UCB-based algorithms maintain an upper bound on the estimated dynamic $\hat{\theta}$.
Therefore, it may violate the condition $\hat{\theta}\in [0,1]$ and return a suboptimal policy.

\section{Conclusions and Future Works}\label{SecCF}
This paper proposed a novel federated online RMAB framework that combines the RMAB and FL paradigms to tackle the cooperative resource allocation problem.
Based on the proposed framework, we proposed the FedTSWI algorithm to solve the RMAB problem by employing the federated TS
algorithm to learn the system dynamics and the WI
policy to maximize the agents’ accumulated rewards.
The proposed algorithm proceeded in episodes, during which the central server and the agent iteratively tracked the optimal solution, following the GPI principle.
In addition, we have derived a regret upper bound for the FedTSWI algorithm, which exhibits a fast convergence rate.
Moreover, we evaluated the proposed algorithm on the application of online multi-user multi-channel access.
Numerical results demonstrated the superior performance of the FedTSWI algorithm compared with the baselines.

\com{In the case study, we have considered that each channel only has two states. However, it is promising to extend the proposed algorithm to more general channel states. This would arise in developing more sophisticated and adaptable algorithms, e.g., deriving the closed-form WI in complex wireless environments. In addition, we assume that the considered problem \eqref{SysGol} is weakly communicating in  Assumption 1. However, it may not always hold in real-world applications. It is critical to relax the assumption and conduct experiments with more practical considerations.}

\appendices
\section{Proof of Proposition 1}\label{appendix1}
Mathematically, the belief state is a representation of an agent's knowledge about the true system state,
which can be updated by the Bayesian rule, i.e.,
\begin{equation}\label{AdxA01}
b' (s'|h)  = \frac{O(h|s', a)\sum_{s\in \mathcal{S}} \hat{\theta}_l(s'|s,a)b(s)} {\sum_{s'\in \mathcal{S}}  O(h|s', a)\sum_{s\in \mathcal{S}} \hat{\theta}_l(s'|s,a)b(s)},
\end{equation}
where $O(h|s', a)$ is the conditional observation probability.

When an action is passive ($a=0$), the agent cannot observe any information about this arm from the environment.
This means that one can only guess a sample $h$ from the finite sample space $\mathcal{H}$ with an equal probability.
Thus, we have $O(h|s', a) = 1/|\mathcal{H}|$, which is a constant.
By substituting it into \eqref{AdxA01}, we obtain
\begin{equation}\label{AdxA02}
b' (s')  = \frac{\sum_{s\in \mathcal{S}} \hat{\theta}_l(s'|s,a)b(s)} {\sum_{s'\in \mathcal{S}}  \sum_{s\in \mathcal{S}} \hat{\theta}_l(s'|s,a)b(s)}.
\end{equation}
We see that the numerator is a function of the belief state and the one-step state transition probability.
More importantly, the belief state $b_{t+1} (s')$ will converge to the Markov stationary distribution if an arm is unobserved for a long period.

When action is active ($a=1$), the agent can observe the true state of this arm, i.e., $h=s$.
According to the  definition of the belief state, we have
\begin{equation}\label{AdxA03}
b' (s')  = \mathrm{Pr} (s'|h=s) = \hat{\theta}_l(s'|s, a=1).
\end{equation}
This proof is concluded by combining \eqref{AdxA02} and \eqref{AdxA03}.
$\hfill\blacksquare$

\section{Proof of Lemma 1}\label{appendix3}
In Algorithm 1, we have two stopping criteria.
Most of the stopping is triggered by the first criterion.
For the second stopping criterion,
we define a macro episode with the start time $ t_{c_i}$ as
\begin{equation}\label{AdxC01}\small
  t_{c_{i+1}} = \min \{ t_l > t_{c_i} : \mathds{N}_{t_l}(s,a) > 2^M \mathds{N}_{t_{l-1}}(s,a), \mathrm{for}\ \mathrm{some}\ (s,a)   \}.
\end{equation}
Let $C$ be the number of macro episodes and $\tilde{T}_i = \sum_{l=c_i}^{c_{i+1}-1} T_l$ be the length of the $i$-th macro episode.
Then, within the $i$-th macro episode, we have $T_l = T_{l-1} + 1$ for all $l = c_i, c_i+1,\ldots,c_{i+1}-2$.
Thus, it yields
\begin{equation}\label{AdxC02}
\begin{split}
  \tilde{T}_i &= \sum_{l=c_i}^{c_{i+1}-1} T_l \geq \sum_{j=1}^{c_{i+1}-c_i-1} \left(T_{c_i} + j \right) + T_{c_{i+1}-1}\\
  &\geq \sum_{j=1}^{c_{i+1}-c_i-1}  (j+1) +1 = \frac{1}{2}(c_{i+1} -c_i)(c_{i+1} -c_i+1),
\end{split}
\end{equation}
which indicates that $c_{i+1}-c_i\leq \sqrt{2\tilde{T}_i}$ for all $i = 1,2,\dots,M$.
Based on this, we have
\begin{equation}\label{AdxC03}
  L_T = c_{C+1} -1 = \sum_{i=1}^{C} (c_{i+1} - c_i) \leq \sum_{i=1}^{C} \sqrt{2\tilde{T}_i}.
\end{equation}
Since $\sum_{i=1}^{C} \tilde{T}_i =T$, we obtain
\begin{equation}\label{AdxC04}
  L_T = c_{C+1} -1 \leq \sum_{i=1}^{E} \sqrt{2\tilde{T}_i} \leq \sqrt{C\sum_{i=1}^{C} 2\tilde{T}_i} = \sqrt{2CT},
\end{equation}
where the second inequality holds because of the Cauchy-Schwartz inequality.

The remaining question is how to quantify the number of macro episodes.
In Algorithm 1, the number of times a state pair is selected is doubled whenever the second stopping criteria is triggered.
Define the start times of macro episodes as $t_1 \cup ( \cup_{(s,a)\in \mathcal S\times \mathcal A} \{ t_l: l\in \Lambda_{(s,a)}\} )$, where $\Lambda_{(s,a)}=\{ l\leq L_T:  \mathds{N}_{t_l}(s,a) > 2^M \mathds{N}_{t_{l-1}}(s,a)\}$.
Since the number of visits to a state-action pair is doubled at every $t_l$ such that $l\in \Lambda_{(s,a)}$, we have $(2^M)^{|\Lambda_{(s,a)}|} -1 \leq T$ by using the arithmetic sequence sum formula.
Therefore, the size of $\Lambda_{(s,a)}$ can be bounded as
\begin{equation}\label{AdxC05}
  |\Lambda_{(s,a)}| \leq \frac{1}{M} \log (T+1).
\end{equation}
Based on this property, we obtain an upper bound on the number of macro episodes
\begin{equation}\label{AdxC06}
\begin{split}
  E &\leq 1 + \sum_{(s,a)} |\Lambda_{(s,a)}| \leq 1 + \frac{N|\mathcal S||\mathcal A|}{KM} \log \left(\frac{KM(T+1)}{N|\mathcal S||\mathcal A|} \right)\\
  & \leq \frac{N|\mathcal S||\mathcal A|}{KM} \log \left(T+1 \right) = \frac{2N|\mathcal S|}{KM} \log \left(T+1 \right),
\end{split}
\end{equation}
where the first inequality is the union bound and the third inequality is obtained using the fact that $|\mathcal A|=2$.
In addition,  the second inequality holds because an arm can be uniformly observed with probability $K/N$  and the log function is concave.

Finally, by substituting \eqref{AdxC06} into \eqref{AdxC04}, we can bound the number of episodes by
\begin{equation}\label{AdxC07}
  L_T \leq \sqrt{2CT} \leq \sqrt{ \frac{4N|\mathcal S|T}{KM} \log \left(T+1 \right) }.
\end{equation}
This concludes the proof.
$\hfill\blacksquare$

\section{Proof of Theorem 1}\label{appendix4}
For each arm at episode $l$, we can write the optimal average per-slot reward of the belief MDP at state $(s_t,a_t)$, which satisfies the following Bellman equation, i.e.,
\begin{equation}\label{BellEquSys}
R^{\pi}(s_t,a_t) = \rho(\hat{\theta}_l) + V(s_t, \hat{\theta}_l) - \sum_{s' \in {S}} \hat{\theta}_l(s'|s_t, a_t)V(s', \hat{\theta}_l),
\end{equation}
where $V(s_t, \hat{\theta}_l)$ is the value function when the initial system state is $s_t$ and the system dynamics is $\hat{\theta}_l$.
Then, the expected regret of Algorithm 1 for each arm can be obtained by substituting \eqref{BellEquSys} into the definition of the regret, which yields
\begin{equation}\label{RegDecomp02}\small
\begin{split}
&\mathcal{R}eg(n, T) = \underbrace{T \mathbb{E} \left[\rho(\theta^{\ast}) \right] - \mathbb{E} \left[ \sum_{l=1}^{L}  T_l \rho({\hat{\theta}}_l) \right]}\limits_{\mathcal{R}eg_1}\\
&+ \underbrace{\mathbb{E} \left[\sum_{l=1}^{L} \sum_{t=t_l}^{t_{l+1}-1} \left[V(s_{t+1}, \hat{\theta}_l) -  V(s_t, \hat{\theta}_l)  \right]  \right]}\limits_{\mathcal{R}eg_2}\\
&+ \underbrace{\mathbb{E} \left[ \sum_{l=1}^{L} \sum_{t=t_l}^{t_{l+1}-1}  \left[\sum_{s' \in {S}} \hat{\theta}_l(s'|s,\pi(s))V(s', \hat{\theta}_l)- V(s_{t+1}, \hat{\theta}_l) \right]  \right]}\limits_{\mathcal{R}eg_3}.
\end{split}
\end{equation}

In the following, we derive the expresses of  $\mathcal{R}eg_1$, $\mathcal{R}eg_2$, and $\mathcal{R}eg_3$.
For $\mathcal{R}eg_1$, according to monotone convergence theorem, we have
\begin{equation}\label{AdxD01}
\begin{split}
\mathcal{R}eg_1 &=T\mathbb{E}[\rho (\theta^{\ast})] -  \mathbb{E} \left[ \sum_{l=1}^{\infty} \mathds{1}_{t_l \leq T} T_l \rho (\hat{\theta}_l)  \right] \\
&=T\mathbb{E} [\rho (\theta^{\ast})] - \sum_{l=1}^{\infty} \mathbb{E} \left[  \mathds{1}_{t_l \leq T} T_l \rho (\hat{\theta}_l) \right].
\end{split}
\end{equation}
Based on the first stopping criteria of the FedTSWI algorithm, it can be deduced that $T_l \leq T_{l-1}+1$ for all $l$.
Thus, we obtain
\begin{equation}\label{AdxD02}
  \mathbb{E} \left[  \mathds{1}_{t_l \leq T} T_l \rho (\hat{\theta}_l) \right] \leq \mathbb{E} \left[  \mathds{1}_{t_l \leq T} (T_{l+1}+1) \rho (\hat{\theta}_l) \right].
\end{equation}
Note that $\mathds{1}_{t_l \leq T} (T_{l+1}+1) $ is a measurable random variable.
Since $\mathbb{E}[f(\hat{\theta}_l, X)] = \mathbb{E}[f({\theta^{\ast}}, X)]$  (see Lemma 2 in \cite{ouyang2017learning}), we have
\begin{equation}\label{AdxD03}
\mathbb{E} \left[  \mathds{1}_{t_l \leq T} (T_{l+1}+1) \rho (\hat{\theta}_l) \right] = \mathbb{E} \left[  \mathds{1}_{t_l \leq T} (T_{l+1}+1) \rho (\theta^{\ast}) \right].
\end{equation}
Combining \eqref{AdxD01}, \eqref{AdxD02}, and \eqref{AdxD03}, it yields
\begin{equation}\label{AdxD04}
\begin{split}
\mathcal{R}eg_1 &\leq T\mathbb{E}[\rho (\theta^{\ast})] -  \mathbb{E} \left[  \mathds{1}_{t_l \leq T} (T_{l+1}+1) \rho (\theta^{\ast}) \right]  \\
&=T\mathbb{E} [\rho (\theta^{\ast})] -  \mathbb{E} \left[ \sum_{l=1}^{L_T} (T_{l-1} +1) \rho (\theta^{\ast}) \right]\\
&=\mathbb{E} [L_T \rho (\theta^{\ast})] + \mathbb{E} \left[ \left(T - \sum_{l=1}^{L_T}T_{l-1} \right) \rho (\theta^{\ast}) \right]\\
&\leq \mathbb{E} [L_T],
\end{split}
\end{equation}
where the last inequality holds since $\rho(\theta^{\ast}) \leq 1$ and $\sum_{l=1}^{L_T} T_{l-1} =  T_0 + \sum_{l=1}^{L_T-1} T_l \leq T$.

For $\mathcal{R}eg_2$, we have
\begin{equation}\label{AdxD05}
\begin{split}
  \mathcal{R}eg_2 &= \mathbb{E} \left[\sum_{l=1}^{L_T} \sum_{t=t_l}^{t_{l+1}-1} \left[V(s_{t+1}, \hat{\theta}_l) -  V(s_t, \hat{\theta}_l)  \right]  \right]\\
  &=\mathbb{E} \left[\sum_{l=1}^{L_T}  \left[V(s_{t_l+1}, \hat{\theta}_l) -  V(s_{t_l}, \hat{\theta}_l)  \right]  \right]\\
  &\leq D\mathbb{E} [L_T],
\end{split}
\end{equation}
where the last inequality is obtained using $0\leq V(s, \hat{\theta}_l)\leq SP(\hat{\theta}_l)\leq D$ for all $s$ and $\theta$.

For $\mathcal{R}eg_3$, since $0\leq V(s', \hat{\theta}_l) \leq H$, we have
\begin{equation}\label{AdxD06}\small
\begin{split}
  \mathcal{R}eg_3 & = \mathbb{E} \left[ \sum_{l=1}^{L_T} \sum_{t=t_l}^{t_{l+1}-1}  \left[\sum_{s' \in {S}} \hat{\theta}_l(s'|s,\pi(s))V(s', \hat{\theta}_l)- V(s_{t+1}, \hat{\theta}_l) \right]  \right]\\
     & =\mathbb{E} \left[ \sum_{l=1}^{L_T} \sum_{t=t_l}^{t_{l+1}-1}  \left[ \sum_{s'\in S} \left(\hat{\theta}_l (s'|s,a) - \theta^{\ast} (s'|s,a) \right)V(s', \hat{\theta}_l) \right]  \right] \\
     &\leq D \mathbb{E} \left[ \sum_{l=1}^{L_T} \sum_{t=t_l}^{t_{l+1}-1}  \left[ \left| \sum_{s'\in S} \left(\hat{\theta}_l (s'|s,a) - \theta^{\ast} (s'|s,a) \right) \right| \right]  \right].
\end{split}
\end{equation}
We still need to consider the term in the inner summation, which has the same process as in the Lemma 5 of paper \cite{ouyang2017learning}.
Thus, we have
\begin{equation}\label{AdxD07}
\mathcal{R}eg_3 \leq 49  D|\mathcal S| \sqrt{\frac{2NT\log (2T)}{KM}},
\end{equation}
where the factor $\sqrt{N/K}$ comes from the fact that an arm can be uniformly observed with probability $K/N$.

Finally, by summarizing over the arms and the terms $\mathcal{R}eg_1$, $\mathcal{R}eg_2$, and $\mathcal{R}eg_3$ for all arm $n$,
we can obtain the regret upper bound of the FedTSWI algorithm, i.e.,
\begin{equation}\label{AdxD08}
  \begin{split}
     \mathcal{R}eg (T) & = \sum_{n=1}^{N}   \mathcal{R}eg (n, T)
        =  \sum_{n=1}^{N} \left( \mathcal{R}eg_1 + \mathcal{R}eg_2 + \mathcal{R}eg_3 \right)\\
       &\leq  \sum_{n=1}^{N} \left( (D+1) \mathbb{E} [L_T] + 49  D|\mathcal S^n| \sqrt{\frac{2NT\log (2T)}{KM}} \right).
  \end{split}
\end{equation}
By using Lemma 1, we get
\begin{equation}\label{AdxD09}
\begin{split}
\mathcal{R}eg(T) \leq  &(D+1)\sum_{n=1}^{N} \sqrt{\frac{4NT|\mathcal S^n|\log(T+1)}{KM}} \\
&+ 49D \sqrt{\frac{2NT\log(2T)}{KM}}  \sum_{n=1}^{N}|\mathcal S^n| + \mathcal{O}(\epsilon T),
\end{split}
\end{equation}
which concludes this proof.
$\hfill\blacksquare$

\section{Proof of Proposition 2}\label{appendix2}
We rewrite the Bellman equation in Section 4.3 as follows:
\begin{equation}\label{AdxB01}\small
\left\{\begin{array}{ll}
 V_{\lambda}^{\pi}\left(b; a=0\right) = \sum\limits_{s \in \mathcal{S}} b(s) R^{\pi}(s, a) +\lambda +  V_{\lambda}^{\pi}\left(b'\right), \\
V_{\lambda}^{\pi}\left(b;a=1 \right) = \sum\limits_{s \in \mathcal{S}} b(s) R^{\pi}(s, a) +  \sum\limits_{s \in \mathcal{S}} b(s) V_{\lambda}^{\pi}\left(b'\right),  \\
\end{array}\right.
\end{equation}
where $V_{\lambda}^{\pi}\left(b;a=1 \right)$ is a linear function of the belief state $b$ when taking action $a=1$.
Meanwhile, according to the general result on the convexity of the value function in POMDP \cite{Sondik1978Near},
$V_{\lambda}^{\pi}\left(b;a=0 \right)$ is convex in $b$ when action $a=0$.
Therefore, $V_{\lambda}^{\pi} (b)$ exhibits a piecewise linear or convex structure.
Fig. \ref{BeliState} shows the convex structure of the value function $V_{\lambda}^{\pi} (b)$ for the two-state Markov model.
We see that the intersection $b^{\ast}(\lambda)$ of functions $V_{\lambda}^{\pi}\left(b;a=0 \right)$
and $V_{\lambda}^{\pi}\left(b;a=1 \right)$ divides the belief space into the active set (LHS) and the passive set (RHS).
\begin{figure}[!t]
\centering
\includegraphics[width=2.4in]{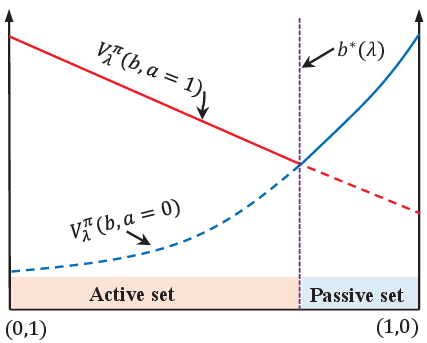}
\caption{The piecewise  convex structure of the maximum value function \eqref{AdxB01} in two-state Markov model.}
\label{BeliState}
\end{figure}

Next, we investigate the impact of the passive subsidy $\lambda$ since it can directly influence the value of $V_{\lambda}^{\pi}\left(b;a=0 \right)$.
Let $\Delta^{a=1}_{\min}$ and $\Delta^{a=1}_{\max}$ be the minimum and maximum  values of $V_{\lambda}^{\pi}\left(b;a=1 \right)$ for $a=1$ as in Fig. \ref{BeliState}, respectively,
and
\begin{equation}\label{AdxB02}
\left\{\begin{array}{ll}
\Delta^{a=0}_{\min} = \min\limits_b \ \sum\limits_{s \in \mathcal{S}} b(s) R^{\pi}(s, a)  +  V_{\lambda}^{\pi}\left(b'\right) , \\
\Delta^{a=0}_{\max}= \max\limits_b \ \sum\limits_{s \in \mathcal{S}} b(s) R^{\pi}(s, a)  +  V_{\lambda}^{\pi}\left(b'\right),  \\
\end{array}\right.
\end{equation}
for $a=0$.
Note that the above values can be achieved at the corners (e.g., $(0,1)$ and $(1,0)$ in Fig. \ref{BeliState}) of the belief simplex due to the convexity of the value function.
Then, we can deduce that: i) If $\lambda>\Delta^1_{\max}-\Delta^0_{\min}$, it is better to be passive for all belief state and $b^{\ast}(\lambda)<0$;
ii) If $\lambda \leq \Delta^1_{\min}-\Delta^0_{\max}$, it is better to be active for all belief state and $b^{\ast}(\lambda) \geq 1$;
iii) If $\Delta^1_{\min}-\Delta^0_{\max} < \lambda \leq \Delta^1_{\max}-\Delta^0_{\min}$, it is better to be active  when $b<b^{\ast}(\lambda)$ and to be passive when $b\geq b^{\ast}(\lambda)$.
Therefore, the optimal policy for the single-agent RMAB problem with subsidy $\lambda$  is a threshold policy.
$\hfill\blacksquare$

\section{Proof of Corollary 1}\label{appendix5}
From value function $V^{\pi}_{\lambda}(b,a=0)$, we see that it is a convex function of subsidy $\lambda$.
According to Proposition \ref{appendix2},  the threshold $b^{\ast}(\lambda)$ is a linear function of subsidy $\lambda$.
In other words, $b^{\ast}(\lambda)$ increases when $\lambda$ changes from $-\infty$ to $+\infty$.
As a result, there exists an $\lambda^{\ast}$ such that the passive set $\mathcal{P}(\lambda)=\emptyset$.
That is, the passive set $\mathcal{P}(\lambda)$ monotonically increases from $\emptyset$ to the whole state space as $\lambda$ increases from $-\infty$ to $+\infty$.
According to \emph{Definition} 3, the single-agent RMAB problem is indexable since every single-armed bandit process is indexable.
$\hfill\blacksquare$

\com{
\section{Proof of Proposition 3}\label{appendix6}
First, we show how to obtain the expressions (32).
Let $\mathds{T}^j (b(t)) \triangleq \mathrm{Pr} [s(t+j) =1 | b(t)] $ be the $j$-step belief update when the arm is unobserved ($a_t =0$) for $j$ consecutive rounds. According to the definition of belief state in (27), i.e., $b^n(s')= b^n(s)\theta^n_{11} +(1-b^n(s))\theta^n_{01}, \    \mathrm{if}  \ a^n_t =0$, we have
\begin{equation}\label{CS02}
\mathds{T}^j (b(t)) = b(t)\mathds{T}^j (1) +  (1-b(t))\mathds{T}^j (0),
\end{equation}
where $\mathds{T}^j (1)$ is the $j$-step transition probability from $1$ to $1$, and $\mathds{T}^j (0)$ is the $j$-step transition probability from $0$ to $1$. Based on the eigendecomposition of the transition matrix $[\theta_{00}, \theta_{01}; \theta_{10}, \theta_{11}]$, we have
\begin{equation}\label{decp1}
\mathds{T}^j (1) = \frac{\theta_{01} + (1-\theta_{11})(\theta_{11}-\theta_{10})^j}{1+\theta_{01} - \theta_{11}},
\end{equation}
and
\begin{equation}\label{decp2}
\mathds{T}^j (0) = \frac{\theta_{01}(1-(\theta_{11}-\theta_{01})^j)}{1+\theta_{01} - \theta_{11}}.
\end{equation}
By substituting \eqref{decp1} and \eqref{decp2}  into  \eqref{CS02}, we obtain
\begin{equation}\label{CS00A}
\mathds{T}^j (b(t)) =   \frac{\theta_{01} - (\theta_{11}-\theta_{01})^j (\theta_{01} - (1+\theta_{01} -\theta_{11}b(t)))}{1+\theta_{01}-\theta_{11}}.
\end{equation}

Then, we show how to obtain expression (33) based on the property of \eqref{CS00A}.
By analyzing \eqref{CS00A}, we  see that Eq. \eqref{CS00A} is monotonically converge to $b_0 = \theta_{01}/(\theta_{01}+1 - \theta_{11})$ as $j\rightarrow \infty$.
Therefore, we have $b_0 \leq \mathds{T}^j (b(t))\leq \mathds{T}^j (1)$. We first analyze $\mathds{T}^j (1)$, which can be expanded by
\begin{equation}\label{bb}
\mathds{T}^j (1) = \theta_{01} + b(t)(\theta_{11} - \theta_{01}).
\end{equation}
Thus, $\mathds{T}^j (1) \leq \max \{ \theta_{01}, \theta_{11}\}, \ \forall j\geq 1$. In addition, $b_0 = \theta_{01}/(\theta_{01}+1 - \theta_{11}) \geq \min \{ \theta_{01}, \theta_{11}\}$.
To conclude, we have
\begin{equation}\label{CS00B}
\min \{ \theta_{01}, \theta_{11}\} \leq \mathds{T}^j (b) \leq \max \{ \theta_{01}, \theta_{11}\}, \ \forall j\geq 1.
\end{equation}}

\balance
\bibliographystyle{./IEEEtran}
\bibliography{./IEEEabrv,./FRMAB_tswi}

\begin{thebibliography}{10}
\providecommand{\url}[1]{#1}
\csname url@samestyle\endcsname
\providecommand{\newblock}{\relax}
\providecommand{\bibinfo}[2]{#2}
\providecommand{\BIBentrySTDinterwordspacing}{\spaceskip=0pt\relax}
\providecommand{\BIBentryALTinterwordstretchfactor}{4}
\providecommand{\BIBentryALTinterwordspacing}{\spaceskip=\fontdimen2\font plus
\BIBentryALTinterwordstretchfactor\fontdimen3\font minus \fontdimen4\font\relax}
\providecommand{\BIBforeignlanguage}[2]{{%
\expandafter\ifx\csname l@#1\endcsname\relax
\typeout{** WARNING: IEEEtran.bst: No hyphenation pattern has been}%
\typeout{** loaded for the language `#1'. Using the pattern for}%
\typeout{** the default language instead.}%
\else
\language=\csname l@#1\endcsname
\fi
#2}}
\providecommand{\BIBdecl}{\relax}
\BIBdecl

\bibitem{chen2008unified}
W.~Chen, L.~Dai, K.~B. Letaief, and Z.~Cao, ``A unified cross-layer framework for resource allocation in cooperative networks,'' \emph{IEEE Trans. Wireless Commun.}, vol.~7, no.~8, pp. 3000--3012, Aug. 2008.

\bibitem{xie2013optimal}
K.~Xie, J.~Cao, X.~Wang, and J.~Wen, ``Optimal resource allocation for reliable and energy efficient cooperative communications,'' \emph{IEEE Trans. Wireless Commun.}, vol.~12, no.~10, pp. 4994--5007, Oct. 2013.

\bibitem{han2008resource}
Z.~Han, T.~Himsoon, W.~P. Siriwongpairat, and K.~R. Liu, ``Resource allocation for multiuser cooperative ofdm networks: Who helps whom and how to cooperate,'' \emph{IEEE Trans. Veh. Technol.}, vol.~58, no.~5, pp. 2378--2391, May 2008.

\bibitem{li2013efficient}
P.~Li, S.~Guo, W.~Zhuang, and B.~Ye, ``On efficient resource allocation for cognitive and cooperative communications,'' \emph{IEEE J. Sel. Areas Commun.}, vol.~32, no.~2, pp. 264--273, Feb. 2013.

\bibitem{tong2021throughput}
J.~Tong, L.~Fu, and Z.~Han, ``Throughput enhancement of full-duplex csma networks using multiplayer bandits,'' \emph{IEEE Internet Things J.}, vol.~8, no.~15, pp. 11\,807--11\,821, Aug. 2021.

\bibitem{letaief2019roadmap}
K.~B. Letaief, W.~Chen, Y.~Shi, J.~Zhang, and Y.-J.~A. Zhang, ``The roadmap to 6g: Ai empowered wireless networks,'' \emph{IEEE commun. magaz.}, vol.~57, no.~8, pp. 84--90, Aug. 2019.

\bibitem{bubeck2012regret}
S.~{Bubeck} and N.~{Cesa-Bianchi}, ``Regret analysis of stochastic and nonstochastic multi-armed bandit problems,'' \emph{Found. Trends Mach. Learn.}, vol.~5, no.~1, pp. 1--122, Nov. 2012.

\bibitem{whittle1988restless}
P.~Whittle, ``Restless bandits: Activity allocation in a changing world,'' \emph{J. of appl. prob.}, vol.~25, pp. 287--298, 1988.

\bibitem{nino2023markovian}
J.~Ni{\~n}o-Mora, ``Markovian restless bandits and index policies: A review,'' \emph{Mathematics}, vol.~11, no.~7, p. 1639, Jul. 2023.

\bibitem{wang2022optimistic}
K.~Wang, L.~Xu, A.~Taneja, and M.~Tambe, ``Optimistic {Whittle} index policy: Online learning for restless bandits,'' \emph{arXiv preprint arXiv:2205.15372}, 2022.

\bibitem{van1981certainty}
H.~Van~de Water and J.~Willems, ``The certainty equivalence property in stochastic control theory,'' \emph{IEEE Trans. on Auto. Control}, vol.~26, no.~5, pp. 1080--1087, May 1981.

\bibitem{ouyang2017learning}
Y.~Ouyang, M.~Gagrani, A.~Nayyar, and R.~Jain, ``Learning unknown {Markov} decision processes: A {Thompson} sampling approach,'' in \emph{NeurIPS}, vol.~30, Long Beach, CA, Dec. 2017.

\bibitem{liu2012learning}
H.~Liu, K.~Liu, and Q.~Zhao, ``Learning in a changing world: Restless multiarmed bandit with unknown dynamics,'' \emph{{IEEE} Trans. Inf. Theory}, vol.~59, no.~3, pp. 1902--1916, Mar. 2012.

\bibitem{papadimitriou1987complexity}
C.~H. Papadimitriou and J.~N. Tsitsiklis, ``The complexity of {Markov} decision processes,'' \emph{Math. of operations research}, vol.~12, no.~3, pp. 441--450, Mar. 1987.

\bibitem{gittins2011multi}
J.~Gittins, K.~Glazebrook, and R.~Weber, \emph{Multi-armed bandit allocation indices}.\hskip 1em plus 0.5em minus 0.4em\relax John Wiley \& Sons, 2011.

\bibitem{shao2023survey}
J.~Shao, Z.~Li, W.~Sun, T.~Zhou, Y.~Sun, L.~Liu, Z.~Lin, and J.~Zhang, ``A survey of what to share in federated learning: perspectives on model utility, privacy leakage, and communication efficiency,'' \emph{arXiv preprint arXiv:2307.10655}, 2023.

\bibitem{liu2022hierarchical}
L.~Liu, J.~Zhang, S.~Song, and K.~B. Letaief, ``Hierarchical federated learning with quantization: Convergence analysis and system design,'' \emph{IEEE Trans. Wireless Commun.}, vol.~22, no.~1, pp. 2--18, Jan. 2022.

\bibitem{mcmahan2017communication}
B.~McMahan, E.~Moore, D.~Ramage, S.~Hampson, and B.~A. y~Arcas, ``Communication-efficient learning of deep networks from decentralized data,'' in \emph{Int. Conf. Art. intel. and sta.}, Ft. Lauderdale, Florida, USA, Oct. 2017, pp. 1273--1282.

\bibitem{SuttonRL2018}
R.~S. {Sutton} and A.~G. {Barto}, \emph{Reinforcement Learning: An Introduction}.\hskip 1em plus 0.5em minus 0.4em\relax MIT press, 2018.

\bibitem{letaief2009cooperative}
K.~B. Letaief and W.~Zhang, ``Cooperative communications for cognitive radio networks,'' \emph{Proceedings of the IEEE}, vol.~97, no.~5, pp. 878--893, May 2009.

\bibitem{shi2021federated}
C.~Shi and C.~Shen, ``Federated multi-armed bandits,'' in \emph{Proc. of the AAAI Conf. on Art.l Intel.}, vol.~35, no.~11, Virtual, Feb. 2022.

\bibitem{yang2020federated}
L.~Yang, B.~Tan, V.~W. Zheng, K.~Chen, and Q.~Yang, ``Federated recommendation systems,'' \emph{Federated Learning: Privacy and Incentive}, pp. 225--239, 2020.

\bibitem{xia2020multi}
W.~Xia, T.~Q. Quek, K.~Guo, W.~Wen, H.~H. Yang, and H.~Zhu, ``Multi-armed bandit-based client scheduling for federated learning,'' \emph{IEEE Trans. Wireless Commun.}, vol.~19, no.~11, pp. 7108--7123, Nov. 2020.

\bibitem{li2022privacy}
T.~Li and L.~Song, ``Privacy-preserving communication-efficient federated multi-armed bandits,'' \emph{IEEE J. Sel. Areas Commun.}, vol.~40, no.~3, pp. 773--787, Mar. 2022.

\bibitem{ortner2014regret}
R.~Ortner, D.~Ryabko, P.~Auer, and R.~Munos, ``Regret bounds for restless {Markov} bandits,'' \emph{Theor. Comput. Sci.}, vol. 558, pp. 62--76, 2014.

\bibitem{weber1990index}
R.~R. Weber and G.~Weiss, ``On an index policy for restless bandits,'' \emph{J. Appl. Prob.}, vol.~27, no.~3, pp. 637--648, Mar. 1990.

\bibitem{liu2010indexability}
K.~Liu and Q.~Zhao, ``Indexability of restless bandit problems and optimality of whittle index for dynamic multichannel access,'' \emph{IEEE Trans. Info. Theory}, vol.~56, no.~11, pp. 5547--5567, Nov. 2010.

\bibitem{ortner2012regret}
R.~Ortner, D.~Ryabko, P.~Auer, and R.~Munos, ``Regret bounds for restless {Markov} bandits,'' in \emph{Int.l Conf. on Alg. Learning Theory}, Lyon, France, Oc. 2012, pp. 214--228.

\bibitem{wang2020restless}
S.~Wang, L.~Huang, and J.~Lui, ``Restless-{UCB}, an efficient and low-complexity algorithm for online restless bandits,'' \emph{NeurIPS}, vol.~33, pp. 11\,878--11\,889, Virtual, Dec. 2020.

\bibitem{akbarzadeh2022learning}
N.~Akbarzadeh and A.~Mahajan, ``On learning {Whittle} index policy for restless bandits with scalable regret,'' \emph{arXiv preprint arXiv:2202.03463}, 2022.

\bibitem{jung2019regret}
Y.~H. Jung and A.~Tewari, ``Regret bounds for {Thompson} sampling in episodic restless bandit problems,'' \emph{NeurIPS}, vol.~32, Vancouver, BC, Canada, Dec. 2019.

\bibitem{russo2018tutorial}
D.~J. Russo, B.~Van~Roy, A.~Kazerouni, I.~Osband, Z.~Wen \emph{et~al.}, ``A tutorial on {Thompson} sampling,'' \emph{Found. Trends Mach. Learn.}, vol.~11, no.~1, pp. 1--96, Jan. 2018.

\bibitem{lalitha2021bayesian}
A.~Lalitha and A.~Goldsmith, ``Bayesian algorithms for decentralized stochastic bandits,'' \emph{IEEE J. Sel. Areas Inf. Theory}, vol.~2, no.~2, pp. 564--583, Feb. 2021.

\bibitem{tong2022age}
J.~Tong, L.~Fu, and Z.~Han, ``Age-of-information oriented scheduling for multichannel {IoT} systems with correlated sources,'' \emph{{IEEE} Trans. Wireless Commun.}, vol.~21, no.~11, pp. 9775--9790, Nov. 2022.

\bibitem{akbarzadeh2022partially}
N.~Akbarzadeh and A.~Mahajan, ``Partially observable restless bandits with restarts: indexability and computation of {Whittle} index,'' in \emph{IEEE 61st Conf. on Decision and Control (CDC)}, Virtual, Dec. 2022, pp. 4898--4904.

\bibitem{tong2018cooperative}
J.~Tong, M.~Jin, Q.~Guo, and Y.~Li, ``Cooperative spectrum sensing: A blind and soft fusion detector,'' \emph{IEEE Trans. Wireless Commun.}, vol.~17, no.~4, pp. 2726--2737, Apr. 2018.

\bibitem{auer2008near}
P.~Auer, T.~Jaksch, and R.~Ortner, ``Near-optimal regret bounds for reinforcement learning,'' in \emph{NeurIPS}, vol.~21, Vancouver, B.C., Canada, Nov. 2008.

\bibitem{Sondik1978Near}
J.~Sondik, Edward, ``The optimal control of partially observable {Markov} processes over the infinite horizon: Discounted costs,'' \emph{Operations research}, vol.~26, no.~2, pp. 282--304, Fed. 1978.

\end{thebibliography}
\end{document}